\documentclass[12pt,reqno,a4paper]{article}
\usepackage{mathrsfs}
\usepackage{amsmath}
\usepackage{amsthm}
\usepackage{amsfonts}
\usepackage{amssymb}
\usepackage[dvips]{graphicx}
\usepackage[dvips]{color}
\usepackage{latexsym}
\usepackage{enumerate}
\usepackage{xspace}
\usepackage[linesnumbered,ruled,vlined]{algorithm2e}
\usepackage{caption}

\DeclareMathOperator*{\argmin}{argmin}
\newtheorem{thm}{Theorem}
\newtheorem{theorem}[thm]{Theorem}

\newtheorem{proposition}[thm]{Proposition}

\SetKwInput{KwInput}{Input}                %

\SetKwInput{KwInitialization}{Initialization} 
\SetKwInput{KwOutput}{Output}              %
\SetKwInput{KwCondition}{Condition}

\SetKwRepeat{Do}{do}{while}%

\begin{document}
\title{Sparse Deep Learning Models with the $\ell_1$ Regularization}
\author{Lixin Shen\thanks{Department of Mathematics, Syracuse University, Syracuse, NY 13244, United States of America. E-mail address: {\it lshen03@syr.edu}. }, \quad Rui Wang\thanks{School of Mathematics, Jilin University, Changchun 130012, P. R. China. E-mail address: {\it rwang11@jlu.edu.cn}. },\quad  \ Yuesheng Xu\thanks{
Department of Mathematics and Statistics, Old Dominion University, Norfolk, VA 23529, USA. This author is also a Professor Emeritus of Mathematics, Syracuse University, Syracuse, NY 13244, USA. E-mail address: {\it y1xu@odu.edu.} All correspondence should be sent to this author.}\quad and \ Mingsong Yan \thanks{ Department of Mathematics and Statistics, Old Dominion University, Norfolk, VA 23529, USA.  E-mail address: {\it myan007@odu.edu}.} }

\date{}

\maketitle{}

\begin{abstract}
Sparse neural networks are highly desirable in deep learning in reducing its complexity. The goal of this paper is to study how choices of regularization parameters influence the sparsity level of learned neural networks. We first derive the $\ell_1$-norm sparsity-promoting deep learning models including single and multiple regularization parameters models, from a statistical viewpoint. We then characterize the sparsity level of a regularized neural network in terms of the choice of the regularization parameters. Based on the characterizations, we develop iterative algorithms for selecting regularization parameters so that the weight parameters of the resulting deep neural network enjoy prescribed sparsity levels. Numerical experiments are presented to demonstrate the effectiveness of the proposed algorithms in choosing desirable regularization parameters and obtaining corresponding neural networks having both of predetermined sparsity levels and satisfactory approximation accuracy. 
\end{abstract}

\section{Introduction}
The last decade has witnessed remarkable advancements of
deep learning. Mathematically, the success is due to the expressiveness of deep neural networks (DNNs). DNNs are compositions of multiple layers, each applying a linear transformation (comprising a weight matrix and a bias vector) followed by a nonlinear activation function. The richness of parameters empowers deep learning in addressing diverse practical challenges. However, this strength also introduces the risk of overfitting, particularly in scenarios with limited training data.  The classical $\ell_2$ regularization technique has proven effective in mitigating overfitting when training DNNs \cite{krogh1991simple, krizhevsky2012imagenet}. In addition to the overfitting challenge,  the substantial number of parameters in DNNs presents issues related to memory usage and computational burden. To address these issues simultaneously, recent research endeavors seek to develop more compact networks with significantly fewer parameters. The term a {\it sparse DNN} refers to a deep neural network in which a significant portion of their weight parameters are zero. It was emphasized in \cite{hoefler2021future} that the future of deep learning lies in sparsity. For a comprehensive review of recent sparsity techniques used in deep learning, readers are referred to \cite{hoefler2021sparsity}.  Sparse DNNs can save tremendous computing time compared to {\it dense} DNNs, when using trained predictors to make decisions. These techniques encompass model compression \cite{ba2014deep, hinton2015distilling}, pruning unnecessary connections \cite{blalock2020state, han2015learning}, and the addition of a sparse regularization term, a main focus of this paper. A mathematical definition on the sparsity promoting functions was discussed in \cite{Shen-Suter-Tripp:JOTA:2019}.

Various sparsity-promoting regularization techniques have been employed in deep learning to enhance sparsity of DNNs.
The $\ell_1$ and $\ell_0$ regularizations were integrated in \cite{collins2014memory} to promote sparsity of parameters of convolutional neural networks (CNNs). A variety of regularization terms, such as group lasso, %
exclusive group lasso, %
and their variations, %
were incorporated into the training of neural networks in \cite{alvarez2016learning, scardapane2017group, wen2016learning, yoon2017combined, zhou2016less} to promote structured sparsity in the resulting DNNs or CNNs. A non-convex transformed $\ell_1$ sparse regularization was introduced in 
\cite{ma2019transformed}, going beyond convex regularization terms, into deep learning. The $\ell_1$-norm regularization was used in \cite{xu2023sparse} to train DNNs for solving nonlinear partial differential equations, demonstrating good generalization results.
Numerical results showed in the above mentioned papers confirm empirically that regularization related to the $\ell_1$-norm can promote sparsity of a regularized DNN and at the same time preserve approximation accuracy. 
However, what absents is a strategy for the choice of the regularization parameter, which guarantees a predetermined sparsity level of the resulting DNN.

It is well-known \cite{tikhonov1977solutions} that an $\ell_2$-norm regularization, such as the Tikhonov regularization, can mitigate the ill-posedness of a system under consideration and suppress noise contaminated in given data, see also \cite{lu2013regularization}. An algorithm for solving optimization problems involving an $\ell_0$ sparse regularization was studied in \cite{fang2024}, which numerically presents the ability of the $\ell_0$ regularization both suppressing noise and promoting sparsity of solutions. Moreover, it was shown in \cite{grasmair2008sparse} that an $\ell_p$-norm regularization can be used to suppress noise. The technique developed in \cite{grasmair2008sparse} was employed in \cite{liu2023parameter} to craft a regularization parameter choice strategy for the $\ell_1$-norm regularization leading to a sparse solution having sparsity of a prescribed level and accuracy of certain order, which were proved rigorously and validated experimentally. It is the goal of this paper to devise a regularization parameter choice strategy for the $\ell_1$-norm regularization that allows to achieve sparsity of DNNs at a prescribed level. We first derive regularization models by employing a maximum a posteriori probability (MAP) estimate under prior assumptions on the weight matrices and bias vectors of the DNNs. Unlike the optimization problems investigated in \cite{liu2023parameter}, which are convex, those to be studied in this paper will involve DNNs, and thus, they are inherently highly non-convex. 

Developing a regularization parameter choice strategy for a highly non-convex fidelity term is a challenging task. To address this challenge, we utilize the generalized Fermat rule to characterize a local minimizer of a non-convex function. %
Note that a DNN has a multiple layers, with different layers representing different scales of a decision function to be learned. Therefore, it is not effective to impose a single regularization parameter for all weight matrices of different layers. It was demonstrated numerically in \cite{xu2023sparse} that imposing different regularization parameters for weight matrices in different layers leads to a more accurate learning solution than imposing a single parameter. Compared to single-parameter regularization, the multi-parameter $\ell_1$-norm regularization offers greater flexibility in promoting sparsity in DNNs and reduces the sensitivity of the model from the variation of the single regularization parameter.  This motivates us to study a parameter choice strategy for the multi-parameter $\ell_1$-norm regularization for learning a DNN. Specifically, we characterize how the regularization parameter in each layer influences the sparsity of the weight matrix of the layer.
Based on this characterization, we then propose an iterative strategy to obtain a DNN with weight parameters that achieve a desired sparsity level in each layer while maintaining comparable accuracy.

We organize this paper in seven sections. In Section \ref{section: sparse deep learning with l1 norm}, we derive the sparse deep learning with single regularization parameter from a statistical viewpoint. In Section \ref{section: parameter choice for single, sparsity and accuracy}, we characterize the relation between the regularization parameter and the sparsity of weight matrices in DNNs. We develop an iterative algorithm to select regularization parameter $\lambda$ such that the resulting neural network has a prescribed sparsity level. In Section \ref{section: multi parameter model}, we extend the results from a model with a single regularization parameter to one with multiple regularization parameters, and derive an iterative strategy to determine regularization parameters such that the weight matrices for each layer enjoy given sparsity levels. In Section \ref{ImplementationIssues}, an algorithm for solving non-convex optimization problems with the $\ell_1$-norm regularization is discussed. In Section \ref{section: numerical experiment}, we implement the iterative strategy introduced in Sections \ref{section: parameter choice for single, sparsity and accuracy} and  \ref{section: multi parameter model} for choosing regularization parameters and validate its effectiveness.  We make conclusions in Section \ref{section: conclusion}.

\section{Sparse Deep Learning with the $\ell_1$ Regularization}\label{section: sparse deep learning with l1 norm}
In this section, we derive sparse deep learning regularization models with a single $\ell_1$ regularization term from a statistical viewpoint. 

We begin by recalling the notation of DNNs. A fully connected feed-forward neural network (FNN) of depth $D\in\mathbb{N}$ comprises an input layer, $D-1$ hidden layers, and an output layer. An FNN with more than two hidden layers is typically referred to as a DNN. For $p,q\in\mathbb{N}$, a DNN is a vector-valued function from $\mathbb{R}^p$
to $\mathbb{R}^q$ 
formed by compositions of functions, each defined by an activation function applied component-wise to an affine map. Specifically, for each $n\in\mathbb{N}$, we define $\mathbb{N}_n:=\{1,2,\ldots,n\}$ with $\mathbb{N}_0:=\emptyset$. Let $\sigma: \mathbb{R}\to\mathbb{R}$ be a given activation function. We then  define a vector-valued function as follows: 
\begin{equation*}\label{activationF}
\sigma({x}):=[\sigma(x_j):j\in\mathbb{N}_n]^\top \ \ \mbox{for}\ \ {x}:=[x_j: j\in\mathbb{N}_n]^\top\in\mathbb{R}^n.
\end{equation*}
For $n$ vector-valued functions $f_j$, $j\in\mathbb{N}_n$, where the range of $f_j$ is contained in the domain of $f_{j+1}$, for $j\in\mathbb{N}_{n-1}$, we denote the consecutive composition of $f_j$, $j\in\mathbb{N}_n$, by
\begin{equation*}\label{consecutive_composition}
    \bigodot_{j=1}^n f_j:=f_n\circ f_{n-1}\circ\cdots\circ f_2\circ f_1,
\end{equation*}
whose domain is that of $f_1$.
For each $k\in\mathbb{N}_{D-1}$, we denote by $n_k$ the number of neurons of the $k$-th layer and set $n_0:=p$, $n_D:=q$. For each $k\in\mathbb{N}_D$, we denote by $W^k\in\mathbb{R}^{n_k\times n_{k-1}}$ the weight matrix and $b^k\in\mathbb{R}^{n_k}$ the bias vector of the $k$-th layer. We let $\Theta:=(W^k,b^k)_{k\in\mathbb{N}_D}$ be the collection of trainable parameters in DNNs. 
 For given parameters $\Theta:=(W^k,b^k)_{k\in\mathbb{N}_D}$, a DNN is a function defined by
\begin{equation}\label{DNN}
\mathcal{N}_\Theta(x):=\left(W^D\bigodot_{k=1}^{D-1} \sigma(W^k\cdot+b^k)+b^D\right)({x}),\ \ {x}\in\mathbb{R}^p.
\end{equation}
The DNN defined by \eqref{DNN} refers to the fully connected neural network, and it has been pointed out in \cite{xu2022convergence} that the convolutional neural network (CNN) is a special case of \eqref{DNN},  with the weight matrix of the convolutional layer being a Toeplitz matrix.  

For simplicity, for each $k\in\mathbb{N}_D$, we denote by $w^k$ the vector obtained from vectorizing the weight matrix $W^k$, and by $d_k:=n_kn_{k-1}$ the length of $w^k$. In this notation, the trainable parameters $\Theta$ can be represented as $(w^k,b^k)_{k\in\mathbb{N}_D}$. We further concatenate the weight parameters $w^k$, $k\in\mathbb{N}_D$ as $w:=[w^k:k\in\mathbb{N}_D]$ and the bias parameters $b^k$, $k\in\mathbb{N}_D$ as $b:=[b^k:k\in\mathbb{N}_D]$, and thereby we may rewrite $\Theta$ as $(w,b)$. We set $d_W:=\sum_{k\in\mathbb{N}_D}d_k$ and $d_b:=\sum_{k\in\mathbb{N}_D}n_k$, which represent the total number of entries of $w$ and $b$, respectively, and let $t:=d_W+d_b$ denote the total number of the trainable parameters in the DNN. 

When training a neural network from a given dataset, regularization becomes essential to suppress noise, as observed data are inevitably corrupted with noise. Here, we derive a single-parameter regularization model from a Bayesian viewpoint. A maximum a posteriori probability (MAP) estimate is a way to obtain an estimate of the unknown parameters by maximizing the posterior probability density function \cite{stuart2010inverse}. Suppose that a training dataset $\left\{\left(x^i,y^i\right)\in\mathbb{R}^p\times\mathbb{R}^q: i\in\mathbb{N}_N\right\}$ is given, where the observed labels $y^i,$ $i\in\mathbb{N}_N$ are corrupted with Gaussian noise. For each $i\in\mathbb{N}_N$, we assume the true label, unaffected by noise, of $x^i$ to be $\hat{y}^i\in\mathbb{R}^q$. We denote by $\mathrm{Gaussian} (\mu,v^2)$ the  Gaussian distribution with mean $\mu$ and variance $v^2$. The probability density function of a random variable $x$ following $\mathrm{Gaussian} (\mu,v^2)$ is given by 
\begin{equation}\label{pdf of Gaussian}
    p(x):=\frac{1}{\sqrt{2\pi}v}\exp\left({-\frac{(x-\mu)^2}{2v^2}}\right),\ x\in\mathbb{R}.
\end{equation}
We assume that noises $\epsilon_j^i$, $(i,j)\in\mathbb{N}_N \times \mathbb{N}_q$, in the given data follow independently $\mathrm{Gaussian}(0,v^2)$ with $v>0$, and for each $i\in\mathbb{N}_N$, we let
\begin{equation}\label{epsilon j tilde}
\widetilde{y}^i:=[\widetilde{y}^i_j:j\in\mathbb{N}_q]\in\mathbb{R}^q,    \ \ \mbox{with}\ \ \widetilde{y}^i_j:=\hat{y}_j^i+\epsilon_j^i.
\end{equation}
For each $i\in \mathbb{N}_N$, $\widetilde{y}^i$ is a random variable. We may approximate the data $\hat{y}_j^i$ by $\left(\mathcal{N}_{\Theta}(x^i)\right)_j$ through the model
\begin{equation}\label{epsilon j}
\widetilde{y}^i_j=\left(\mathcal{N}_{\Theta}(x^i)\right)_j+\epsilon_j^i, \ \ j\in\mathbb{N}_q.
\end{equation}
Here, the parameters in $\Theta$ are random variables following a prior distribution. %

For the purpose of obtaining sparse DNNs, we assume that the prior distribution that $\Theta$ follows 
the Laplace distribution. %
We will use $\mathrm{Laplace} (\mu,s)$ to denote the Laplace distribution with location parameter $\mu$ and scale parameter $s>0$. The probability density function of a random variable $x$ following $\mathrm{Laplace} (\mu,s)$ is given by 
\begin{equation}\label{pdf of Laplace}
    p(x)=\frac{1}{2s}\mathrm{exp}\left(-\frac{|x-\mu|}{s}\right), \ x\in\mathbb{R}.
\end{equation}
We impose a prior Laplace distribution with location parameter $\mu$ being $0$ on both the weight parameters $w^k$, $k\in\mathbb{N}_D$ and the bias parameters $b^k$, $k\in\mathbb{N}_D$ for believing that they could be sparse. 

We will review the MAP estimate from Bayesian statistics. Let $Y:=[y^i_j:i\in\mathbb{N}_N,j\in\mathbb{N}_q]\in\mathbb{R}^{q\times N}$ represent the labels from given data, and $\widetilde{Y}:=[\widetilde{y}^i_j:i\in\mathbb{N}_N,j\in\mathbb{N}_q]\in\mathbb{R}^{q\times N}$ denote the random variable. By $p(\Theta|\widetilde{Y}=Y)$ we denote the posterior probability, the probability that $\Theta$ occurs when the random variable $\widetilde{Y}$ takes the value $Y$. The MAP estimate $\Theta^*$ can be obtained by maximizing the posterior probability $p(\Theta|\widetilde{Y}=Y)$. That is, $\Theta^*$ is the solution of the optimization problem  
\begin{equation}\label{MAP estimate}
    \max\left\{p(\Theta|\widetilde{Y}=Y):\Theta:=(w,b)\in\mathbb{R}^t\right\}.
\end{equation}
The posterior probability $p(\Theta|\widetilde{Y}=Y)$ can be computed by using the Bayes theorem.
We denote by $p(\widetilde{Y}=Y)$ the probability of the random variable $\widetilde{Y}$ taking value of $Y$, by $p(\Theta)$ the prior distribution assumed on $\Theta$, and by $p(\widetilde{Y}=Y|\Theta)$ the conditional probability that $\widetilde{Y}=Y$ occurs given $\Theta$ is known. The Bayes theorem \cite{bayes1763lii} states that the posterior distribution $p(\Theta|\widetilde{Y}=Y)$ can be written as 
\begin{equation}\label{Bayes'}
p(\Theta|\widetilde{Y}=Y)=\frac{p(\widetilde{Y}=Y|\Theta)p(\Theta)}{p(\widetilde{Y}=Y)}.
\end{equation}
Substituting equation \eqref{Bayes'} into the object function of problem \eqref{MAP estimate} with noting that the term $p(\widetilde{Y}=Y)$ is a constant with respect to $\Theta$, we rewrite problem \eqref{MAP estimate} as 
\begin{equation*}\label{MAP estimate1}
    \max\left\{p(\widetilde{Y}=Y|\Theta)p(\Theta):\Theta:=(w,b)\in\mathbb{R}^t\right\},
\end{equation*} 
which is further equivalent to 
\begin{equation}\label{optimization problem two ln}
    \max\left\{\log(p(\widetilde{Y}=Y|\Theta))+\log(p(\Theta)):\Theta:=(w,b)\in\mathbb{R}^t, p(\Theta)>0\right\}.
\end{equation}
See, \cite{krol2012preconditioned} for more details. 

We now derive the single-parameter regularization model for regression through the MAP estimate. 

\begin{proposition}\label{prop: square loss single lambda with bias regularized}
Suppose that $\left\{\left(x^i,y^i\right)\in\mathbb{R}^p\times\mathbb{R}^q: i\in\mathbb{N}_N\right\}$ is a given dataset, and the labels $y^i$ are the observed values of the random variables $\tilde{y}^i$ defined by equation \eqref{epsilon j tilde}  with $\epsilon_j^i$, $i\in\mathbb{N}_N$, $j\in\mathbb{N}_q$, independently following $\mathrm{Gaussian} (0,v^2)$ with $v>0$ and DNNs $\mathcal{N}_{\Theta}$ satisfy equation \eqref{epsilon j} for parameters $\Theta:=(w,b)$ with $w\in\mathbb{R}^{d_W}$, $b\in\mathbb{R}^{d_b}$. If for each $k\in\mathbb{N}_D$, $w_j^k$, $j\in\mathbb{N}_{d_k}$ and $b_{j}^k$, $j\in\mathbb{N}_{n_k}$ independently follow $\mathrm{Laplace} (0,s)$ with $s>0$, then the MAP estimate $\Theta^*$ is a solution of the optimization problem
\begin{equation}\label{single lambda regression model with bias regularized} 
    \min
    \left\{\sum_{i\in\mathbb{N}_N}\left\|\mathcal{N}_{\Theta}(x^i)-y^i\right\|_2^2+ \lambda\left(\|w\|_1+\|b\|_1\right):\Theta:=(w,b)\in\mathbb{R}^t\right\}
\end{equation}
with $\lambda:=2v^2/s$.
\end{proposition}
\begin{proof}
Let $Y:=[y^i_j:i\in\mathbb{N}_N,j\in\mathbb{N}_q]\in\mathbb{R}^{q\times N}$ and consider the random variable $\widetilde{Y}:=[\widetilde{y}^i_j:i\in\mathbb{N}_N,j\in\mathbb{N}_q]\in\mathbb{R}^{q\times N}$. It suffices to show that problems \eqref{optimization problem two ln} and \eqref{single lambda regression model with bias regularized} are equivalent. To this end, we compute the two probabilities  $\log(p(\widetilde{Y}=Y|\Theta))$ and $\log(p(\Theta))$. 

We first compute  $\log(p(\widetilde{Y}=Y|\Theta))$. It follows from equation \eqref{epsilon j} with $\epsilon_j^i$, $i\in\mathbb{N}_N$, $j\in\mathbb{N}_q$, independently following $\mathrm{Gaussian}(0,v^2)$ with $v>0$ that $\widetilde{y}^i_j$, $i\in\mathbb{N}_N$, $j\in\mathbb{N}_q$, independently follow $\mathrm{Gaussian}(\left(\mathcal{N}_{\Theta}(x^i)\right)_j,v^2)$. Since random variables $\widetilde{y}^i_j$, $i\in\mathbb{N}_N$, $j\in\mathbb{N}_q$ are independent, we have that 
\begin{equation*}
p(\widetilde{Y}=Y|\Theta)=\prod_{i\in\mathbb{N}_N}\prod_{j\in\mathbb{N}_q} p(\widetilde{y}^i_j=y^i_j|\Theta),
\end{equation*}
which together with probability density function \eqref{pdf of Gaussian} of Gaussian distribution leads to 
\begin{equation*}
p(\widetilde{Y}=Y|\Theta)=\prod_{i\in\mathbb{N}_N}\prod_{j\in\mathbb{N}_q} \frac{1}{\sqrt{2 \pi}v} \exp \left(-\frac{\left(\left(\mathcal{N}_{\Theta}(x^i)\right)_j-y^i_j\right)^2}{2 v^2}\right).
\end{equation*}
Therefore, we obtain that 
\begin{align}
\log(p(\widetilde{Y}=Y|\Theta))%
&=-qN\log\left(\sqrt{2\pi}v\right)-\frac{1}{2 v^2}\sum_{i\in\mathbb{N}_N}\sum_{j\in\mathbb{N}_q} \left(\left(\mathcal{N}_{\Theta}(x^i)\right)_j-y^i_j\right)^2.\nonumber
\label{ln p(y|Theta)}
\end{align}
Let $C_1:=-qN\log\left(\sqrt{2\pi}v\right)$ and note that $C_1$ is a constant independent of $\Theta$. We then obtain that
\begin{equation}\label{ln p(y|Theta)}
\log(p(\widetilde{Y}=Y|\Theta))=C_1-\frac{1}{2v^2}\sum_{i\in\mathbb{N}_N}\left\|\mathcal{N}_{\Theta}(x^i)-y^i\right\|_2^2.
\end{equation}

We next compute $\log(p(\Theta))$. Since for each $k\in\mathbb{N}_D$, $w^k_j$, $j\in\mathbb{N}_{d_k}$ and $b_{j'}^k$, $j'\in\mathbb{N}_{n_k}$ are independent, by the definition of $\Theta$, we have that
\begin{equation}\label{w and b are independent}
p(\Theta)=\prod_{k\in\mathbb{N}_D}\left(\prod_{j'\in\mathbb{N}_{n_k}}p(b^k_{j'})\right)\left(\prod_{j\in\mathbb{N}_{d_k}} p(w^k_j)\right). 
\end{equation}
This combined with the assumption that for each $k\in\mathbb{N}_D$, $w^k_j$, $j\in\mathbb{N}_{d_k}$ and $b_{j'}^k$, $j'\in\mathbb{N}_{n_k}$ follow $\mathrm{Laplace} (0,s)$ with $s>0$, and the probability density function \eqref{pdf of Laplace} of the Laplace distribution yields that 
\begin{equation*}
p(\Theta)=\prod_{k\in\mathbb{N}_D}\left(\prod_{j'\in\mathbb{N}_{n_k}}\frac{1}{2s}\mathrm{exp}(-|b^k_{j'}|/s)\right)\left(\prod_{j\in\mathbb{N}_{d_k}}\frac{1}{2s}\mathrm{exp}\left(-|w^k_j|/s\right)\right). 
\end{equation*}
Therefore, we derive that 
\begin{align*}
\log(p(\Theta))%
    &=\sum_{k\in\mathbb{N}_D}\left(
    -(n_k+d_k)\log(2s)-\frac{1}{s}\sum_{j'\in\mathbb{N}_{n_k}}|b^k_{j'}|-\frac{1}{s}\sum_{j\in\mathbb{N}_{d_k}}|w^k_j|\right).
\end{align*}
Noting that $w=[w^k:k\in\mathbb{N}_D]$, $b=[b^k:k\in\mathbb{N}_D]$ and $t=\sum_{k\in\mathbb{N}_D}(n_k+d_k)$, the above equation leads to 
\begin{equation}\label{single ln p Theta with bias regularized}
        \log(p(\Theta))=C_2-\frac{1}{s}(\|w\|_1+\|b\|_1),
\end{equation} 
where $C_2:=-t\log(2s)$ is also a constant independent of $\Theta$. Substituting equations \eqref{ln p(y|Theta)} and \eqref{single ln p Theta with bias regularized} into the object function of the optimization problem \eqref{optimization problem two ln}, and noting that $C_1$ and $C_2$ are constants independent of $\Theta$, we may rewrite problem \eqref{optimization problem two ln} as 
\begin{equation}\label{proof argmax equiv pr with bias regularized}
\max\left\{-\frac{1}{2v^2}\sum_{i\in\mathbb{N}_N}\left\|\mathcal{N}_{\Theta}(x^i)-y^i\right\|_2^2-\frac{1}{s}(\|w\|_1+\|b\|_1):\Theta:=(w,b)\in\mathbb{R}^t\right\}.
\end{equation}
By multiplying $-2v^2$ on each term of the object function of problem \eqref{proof argmax equiv pr with bias regularized} and letting $\lambda:=2v^2/s$, we observe that problem \eqref{proof argmax equiv pr with bias regularized} is equivalent to problem \eqref{single lambda regression model with bias regularized}. Therefore, problems \eqref{optimization problem two ln} and \eqref{single lambda regression model with bias regularized} are equivalent. 
\end{proof}

In the regularization model \eqref{single lambda regression model with bias regularized}, the weight parameters and the bias parameters are both regularized, due to the assumption that both of them follow the Laplace distribution. It was pointed out in \cite{goodfellow2016} that the neural network may suffer from underfitting when the bias parameters are regularized. In fact, the regularization term for the bias parameters in model \eqref{single lambda regression model with bias regularized} is derived from the prior Laplace distribution, which takes a higher probability around the value of zero. This motivates us to assume a uniform distribution on the bias parameters. In other words, we assume that the bias parameters take values with equal probabilities. We use $\mathrm{Uniform} (a,a')$ to denote the uniform distribution over the interval $(a,a')$.  The probability density function of a random variable $x$ following $\mathrm{Uniform} (a,a')$ is given by 
\begin{equation}\label{pdf of Uniform}
    p(x):=\begin{cases}
        \frac{1}{a'-a}, & x\in[a,a'],\\
        0, &\mathrm{otherwise}.
    \end{cases}
\end{equation}
In the next proposition, we derive a regularization problem under the assumption that the bias parameters independently follow $\mathrm{Uniform} (-M,M)$ with $M>0$.

\begin{proposition}\label{prop: square loss single lambda}
Suppose that $\left\{\left(x^i,y^i\right)\in\mathbb{R}^p\times\mathbb{R}^q: i\in\mathbb{N}_N\right\}$ is a given dataset, and the labels $y^i$ are the observed values of the random variables $\tilde{y}^i$ defined by equation \eqref{epsilon j tilde}  with $\epsilon_j^i$, $i\in\mathbb{N}_N$, $j\in\mathbb{N}_q$, independently following $\mathrm{Gaussian} (0,v^2)$ with $v>0$ and DNNs $\mathcal{N}_{\Theta}$ satisfy equation \eqref{epsilon j} for parameters $\Theta:=(w,b)$ with $w\in\mathbb{R}^{d_W}$, $b\in\mathbb{R}^{d_b}$. If for each $k\in\mathbb{N}_D$, $w_j^k$, $j\in\mathbb{N}_{d_k}$ independently follow $\mathrm{Laplace}(0,s)$ with $s>0$, and $b^k_j$, $j\in\mathbb{N}_{n_k}$ independently follow $\mathrm{Uniform}(-M,M)$ with $M>0$, then the MAP estimate $\Theta^*$ is a solution of the following optimization problem
\begin{equation}\label{single lambda regression model}
    \min\left\{\sum_{i\in\mathbb{N}_N}\left\|\mathcal{N}_{\Theta}(x^i)-y^i\right\|_2^2 + \lambda\|w\|_1:\Theta:=(w,b)\in\mathbb{R}^t\right\}
\end{equation}
with $\lambda:=2v^2/s$.%
\end{proposition}

\begin{proof}
Let $Y:=[y^i_j:i\in\mathbb{N}_N,j\in\mathbb{N}_q]\in\mathbb{R}^{q\times N}$ and consider the random variable $\widetilde{Y}:=[\widetilde{y}^i_j:i\in\mathbb{N}_N,j\in\mathbb{N}_q]\in\mathbb{R}^{q\times N}$. We have known that the MAP estimate $\Theta^*$ is the solution of the optimization problem \eqref{optimization problem two ln}. Since $\epsilon_j^i$, $i\in\mathbb{N}_N$, $j\in\mathbb{N}_q$, independently following $\mathrm{Gaussian}(0,v^2)$ with $v>0$, the same as in Proposition \ref{prop: square loss single lambda with bias regularized}, the term $\log(p(\widetilde{Y}=Y|\Theta))$ has been computed in equation \eqref{ln p(y|Theta)}. 
It remains to compute $\log(p(\Theta))$. Since for each $k\in\mathbb{N}_D$, $w^k_j$, $j\in\mathbb{N}_{d_k}$ independently follow $\mathrm{Laplace}(0,s)$ with $s>0$ and $b^k_j$, $j\in\mathbb{N}_{n_k}$ independently follow $\mathrm{Uniform}(-M,M)$ with $M>0$, we have \eqref{w and b are independent}, which combining with the probability density functions \eqref{pdf of Laplace} and \eqref{pdf of Uniform} of Laplace and uniform distributions, respectively, implies
\begin{equation}\label{Compute:P(Theta)}
    p(\Theta)=\prod_{k\in\mathbb{N}_D}\frac{1}{(2M)^{n_k}}\prod_{j\in\mathbb{N}_{d_k}}\frac{1}{2s}\mathrm{exp}\left(-\frac{|w^k_j|}{s}\right).
\end{equation}
By using equation \eqref{Compute:P(Theta)} and letting 
$$
C_3:=-\log(2M)\sum_{k\in\mathbb{N}_D} n_k - \log(2s)\sum_{k\in\mathbb{N}_D} d_k,
$$
we obtain that 
\begin{equation}\label{single ln p Theta}
        \log(p(\Theta))=C_3-\frac{1}{s}\|w\|_1.
\end{equation} 
Note that the constant $C_3$ is independent of $\Theta$. Combining equations \eqref{ln p(y|Theta)} and \eqref{single ln p Theta}, the optimization problem \eqref{optimization problem two ln} can be rewritten as 
\begin{equation*}\label{proof argmax equiv pr}
    \max\left\{-\frac{1}{2v^2}\sum_{i\in\mathbb{N}_N}\left\|\mathcal{N}_{\Theta}(x^i)-y^i\right\|_2^2- \frac{1}{s}\|w\|_1:\Theta:=(w,b)\in\mathbb{R}^t\right\}, 
\end{equation*}
which is further equivalent to problem \eqref{single lambda regression model} with $\lambda:=2v^2/s$. 
\end{proof}

Note that in Proposition \ref{prop: square loss single lambda}, by assuming a uniform distribution for the bias parameters, we leave them unregularized in model \eqref{single lambda regression model} to prevent potential underfitting. We also remark that the constant $M$ in the uniform distribution assumed on the bias parameters is absorbed by the constant $C_3$ and thus it does not influence the resulting optimization problem \eqref{single lambda regression model}. 
The quantity $\lambda$ appearing in the optimization problem \eqref{single lambda regression model} is the regularization parameter. It is positively correlated with the variance of noise but negatively with the variance of the weight parameters. This suggests that a larger value of $\lambda$ should be chosen when dealing with higher levels of noise in the observed data, aiming to promote more zero entries in the weight parameters.  The fidelity term in problem \eqref{single lambda regression model} is the squared loss function, a common choice for regression.

We next derive the single-parameter regularization model for classification by imposing a different prior distribution on the given data. Let $K\in\mathbb{N}$ and $\left\{\left(x^i,y^i\right)\in\mathbb{R}^p\times\mathbb{N}_K: i\in\mathbb{N}_N\right\}$ be a given dataset. The observed labels $y^i$, $i\in\mathbb{N}_N$, may be mislabeled. For each $i\in\mathbb{N}_N$, similarly to the regression setting, we assume that the true label of $x^i$ is $\hat{y}^i$. For the multi-class classification problems, we may assume that the labels are random variables following the categorical distribution, a {\it discrete} probability distribution whose sample space is the set of $K$ individually identified items. The categorical distribution is the generalization of the Bernoulli distribution to multiple categories. We denote by $\mathrm{Cat}(K,P)$ the categorical distribution with $K$ categories and parameters $P=[P_j: j\in\mathbb{N}_K]$ satisfying $P_j\geq0$, $j\in\mathbb{N}_K$,  where $P_j$ represents the probability of being in $j$-th category and $\|P\|_1=1$. The probability mass function of a random variable $x$ following $\mathrm{Cat}(K,P)$ is given by 
\begin{equation}\label{pmf of cat}
    p(x=j)=P_j, \ j\in\mathbb{N}_K.  
\end{equation}
For each $j\in\mathbb{N}_K$, let $e_j$ denote the unit vector with $1$ for the $j$-th component and 0 otherwise. For each $i\in\mathbb{N}_N$, we introduce a random variable $\widetilde{y}^i$ satisfying 
\begin{equation}\label{classification tilde yi true}
\widetilde{y}^i=\mathrm{Cat}(K,e_{\hat{y}^i}).  
\end{equation}
To describe the prior distribution followed by the given data, we need the softmax function  $\mathcal{S}:\mathbb{R}^K\to[0,1]^K$ for $z=[z_j:j\in\mathbb{N}_K]\in\mathbb{R}^K$ defined by 
\begin{equation}\label{softmax function}
    \left(\mathcal{S}(z)\right)_j:=\frac{e^{z_j}}{\sum_{i\in\mathbb{N}_K} e^{z_i}},\quad j\in\mathbb{N}_K.
\end{equation}
The softmax function defined by \eqref{softmax function} is introduced to normalize the output of neural networks into probabilities for different categories.
Recall that a DNN $\mathcal{N}_\Theta$ defined by \eqref{DNN} with $n_D:=K$ is a vector-valued function from $\mathbb{R}^p$
to $\mathbb{R}^K$. We may approximate $e_{\hat{y}^i}$ in \eqref{classification tilde yi true} by $\mathcal{S}(\mathcal{N}_\Theta(x^i))$ through the model
\begin{equation}\label{classification tilde yi given}
    \widetilde{y}^i=\mathrm{Cat}(K,\mathcal{S}(\mathcal{N}_\Theta(x^i))),
\end{equation}
for parameters $\Theta:=(w,b)$. 

The next proposition presents the single-parameter regularization model by the MAP estimate for classification.

\begin{proposition}\label{prop: single lambda classification}
    Let $K\in\mathbb{N}$. Suppose that $\left\{\left(x^i,y^i\right)\in\mathbb{R}^p\times\mathbb{N}_K: i\in\mathbb{N}_N\right\}$ is a given dataset and the labels $y^i$ are the observed values of the random variables $\tilde{y}^i$ defined by equation \eqref{classification tilde yi true} and DNNs $\mathcal{N}_{\Theta}$ satisfy equation \eqref{classification tilde yi given} for parameters $\Theta:=(w,b)$ with $w\in\mathbb{R}^{d_W}$, $b\in\mathbb{R}^{d_b}$. If for each $k\in\mathbb{N}_D$, $w^k_j$, $j\in\mathbb{N}_{d_k}$, independently follow $\mathrm{Laplace}(0,s)$ with $s>0$, and $b_j^k$, $j\in\mathbb{N}_{n_k}$, independently follow $\mathrm{Uniform}(-M,M)$ with $M>0$, then the MAP estimate $\Theta^*$ is a solution of the following optimization problem
    \begin{equation}\label{single lambda classification model}
    \min\left\{-\sum_{i\in\mathbb{N}_N} \log\left(\left(\mathcal{S}(\mathcal{N}_{\Theta}(x^i))\right)_{y^i}\right) + \lambda\|w\|_1:\Theta:=(w,b)\in\mathbb{R}^t\right\}
    \end{equation}
    with $\lambda:=1/s$. 
\end{proposition}
\begin{proof}
Let $Y:=[y^i:i\in\mathbb{N}_N]\in\mathbb{N}_K^{N}$ and consider the random variable $\widetilde{Y}:=[\widetilde{y}^i:i\in\mathbb{N}_N]\in\mathbb{N}_K^{N}$. As pointed out before, the MAP estimate $\Theta^*$ is the solution of the optimization problem \eqref{optimization problem two ln}. It suffices to compute the two probabilities involved in  \eqref{optimization problem two ln} by using the hypothesis of this proposition. It is noted that for each $k\in\mathbb{N}_D$, $w^k_j$, $j\in\mathbb{N}_{d_k}$ independently follow $\mathrm{Laplace}(0,s)$ with $s>0$, and $b_j^k$, $j\in\mathbb{N}_{n_k}$ independently follow $\mathrm{Uniform}(-M,M)$ with $M>0$. As has been shown in the proof of Proposition \ref{prop: square loss single lambda}, the quantity $\log(p(\Theta))$ can be represented as in equation \eqref{single ln p Theta}. 

It remains to compute $\log(p(\widetilde{Y}=Y|\Theta))$ under the assumptions of this proposition. Since the random variables $\widetilde{y}^i$, $i\in\mathbb{N}_N$ are independent, it follows that
\begin{equation}\label{Computing:Categorical_D}
p(\widetilde{Y}=Y|\Theta)=\prod_{i\in\mathbb{N}_N} p(\widetilde{y}^i=y^i|\Theta).
\end{equation} 
Invoking equation \eqref{classification tilde yi given}
and the probability mass function \eqref{pmf of cat} of the categorical distribution in the right-hand side of \eqref{Computing:Categorical_D} yields
$$
p(\widetilde{Y}=Y|\Theta)=\prod_{i\in\mathbb{N}_N}\left(\mathcal{S}(\mathcal{N}_{\Theta}(x^i))\right)_{y^i}.
$$
Therefore, we obtain that 
\begin{equation}\label{proof log p classification}
\log(p(\widetilde{Y}=Y|\Theta))= \sum_{i\in\mathbb{N}_N} \log\left(\left(\mathcal{S}(\mathcal{N}_{\Theta}(x^i))\right)_{y^i}\right).
\end{equation}
By substituting equations \eqref{single ln p Theta} and \eqref{proof log p classification} into problem \eqref{optimization problem two ln}, we conclue that the MAP estimate $\Theta^*$ can be obtained from solving the optimization problem \eqref{single lambda classification model}. 
\end{proof}

The fidelity term in the regularization problem \eqref{single lambda classification model} is commonly referred to as the cross entropy loss function \cite{murphy2012machine}. Previous studies \cite{polson2017deep, rudner2023function} have demonstrated that the deep learning model with a single regularization parameter can be derived from the Bayesian perspective, where both weights and biases are subject to regularization. However, in our approach, we opt to leave biases unregularized by adopting a uniform distribution as the prior, thereby mitigating the risk of potential underfitting. 

We have derived the single-parameter regularization models \eqref{single lambda regression model} and \eqref{single lambda classification model} by assuming different prior distributions on the given dataset, which lead to different loss functions. Practically, the given samples may be assumed to follow other distributions among various scientific areas. In medical imaging, the Poisson distribution is often assumed as the prior distribution on the observed data, such as positron emission tomography (PET) or the single-photon emission computed tomography (SPECT) data. We denote by $\mathrm{Poisson}(\alpha)$ the Poisson distribution with mean $\alpha$. Let $f$ represent the expected radiotracer distribution, $g$ be the observed projection data, $\gamma$ be the additive counts and $A$ be the SPECT system matrix. Following \cite{krol2012preconditioned}, the emission computed tomography (ECT) model assumes that $g$ follows $\mathrm{Poisson}(Af+\gamma)$ and the resulted fidelity term for MAP estimate is formulated by 
\begin{equation*}
    \mathcal{L}(f):=\langle Af,\mathbf{1}\rangle -\langle \ln(Af+\gamma), g\rangle,  
\end{equation*}
where $\mathbf{1}$ denotes the vector with all components equal to $1$.

In signal processing, the one-bit compressive sensing problem is of interest, where the observed data is assumed to satisfy the two-point distribution for MAP estimate. Let $x$ represent the signal we aim to recover, $y$ be the measurement vector and $B$ be the measurement matrix. For given $x$, \cite{dai2016noisy} assumes that each component of $y$ independently follows a two-point distribution 
\begin{equation*}
    p\left(y_i | x\right)= \begin{cases}1-a, & \text { if } y_i=(\operatorname{sign}(B x))_i, \\ a, & \text { if } y_i=-(\operatorname{sign}(B x))_i,\end{cases}
\end{equation*}
where $a\in(0,1)$ represents the fraction of sign flips. Accordingly, the fidelity term for MAP estimate has the form  
\begin{equation*}
    \mathcal{L}(x):=c\left\|y-\mathrm{sign}(Bx)\right\|_0,
\end{equation*}
where $c:=\log(1-a)-\log(a)$. 

In considering the diverse fidelity terms derived from different assumptions regarding the prior distribution of practical data, we propose to examine a general loss function to study the impact of the regularization parameter on the sparsity of the regularized solutions. We recall that there exist $t$ trainable parameters in the DNN with the form \eqref{DNN}. For the given dataset $\mathcal{D}$, we denote by  $\mathcal{L}_{\mathcal{D}}:\mathbb{R}^t\to\mathbb{R}$ a general loss function of variable $\Theta\in\mathbb{R}^t$. When there is no ambiguity, we shall write $\mathcal{L}$ to replace $\mathcal{L}_{\mathcal{D}}$ for simplicity. The single-parameter regularization model can then be formulated as 
\begin{equation}\label{general single lambda model}
    \min\left\{\mathcal{L}(\Theta) + \lambda\|w\|_1:\Theta:=\left(w,b\right)\in\mathbb{R}^t\right\},
\end{equation}
where $\lambda$ is a positive regularization parameter.  

\section{Choice of the Regularization Parameter} \label{section: parameter choice for single, sparsity and accuracy}
In this section, we delve into the relationship between the choice of the regularization parameter and the sparsity of the regularized solutions of problem \eqref{general single lambda model}. 

We first recall that a vector in $\mathbb{R}^n$ is said to have sparsity of level $l\in\mathbb{Z}_{n+1}:=\{0,1,\ldots,n\}$ if it has exactly $l$ nonzero components. To represent sparsity of vectors in $\mathbb{R}^n$ specifically, we need the sparsity partition of $\mathbb{R}^n$ introduced in \cite{xu2022sparse}.  Recall that $e_j,j\in\mathbb{N}_n,$ denote the canonical basis for $\mathbb{R}^n$. Using these vectors, we define $n+1$ subsets of $\mathbb{R}^n$ as follows: 
\begin{eqnarray*}
\Omega_{n,0}&:=&\{0\},\\ 
\Omega_{n,l}&:=&\left\{\sum_{j\in\mathbb{N}_l}u_{s_j}e_{s_j}
:u_{s_j}\neq 0\ \mathrm{for} \
1\leq s_1<s_2<\cdots< s_l\leq n\right\}, \ \mbox{for}\ l\in\mathbb{N}_n.
\end{eqnarray*}
For each $l\in\mathbb{Z}_{n+1}$, the subset $\Omega_{n,l}$ coincides with the set of all vectors in $\mathbb{R}^n$ having
sparsity of level $l$. 
We define an ordered subset of $\mathbb{N}_{n}$ with cardinality $l\in\mathbb{N}_n$ by
$$
\mathbb{S}_{n,l}:=\{s_i\in \mathbb{N}_{n}: i\in\mathbb{N}_l, \ \mbox{with}\ 1\leq s_1<s_2<\cdots<s_l\leq n\}.
$$
In this notation, for $w\in \Omega_{n,l}$, there exists an ordered subset $\mathbb{S}_{n,l}$ of $\mathbb{N}_{n}$ with cardinality $l$ such that
$$
w=\sum_{i\in \mathbb{S}_{n,l}}w_ie_i, \ \mbox{with}\ w_i\neq 0, \ \mbox{for all}\ i\in \mathbb{S}_{n,l}.
$$

To characterize the local minimizer of problem \eqref{general single lambda model}, we recall the notion of the general subdifferential of a proper function $f:\mathbb{R}^n\to\mathbb{R}\cup\{+\infty\}$, which may not be convex \cite{rockafellar2009variational}. The domain of $f$ is denoted by 
$$
\mathrm{dom}(f):=\left\{x\in\mathbb{R}^n:f(x)<+\infty\right\}.
$$
For each $x\in\mathrm{dom}(f)$, we say that $v\in\mathbb{R}^n$ is a regular subdifferential of $f$ at $x$ if 
\begin{equation*}
    \liminf _{\substack{y \neq x \\ y \rightarrow x}} \frac{1}{\|x-y\|}\left[f(y)-f(x)-\left\langle v, y-x\right\rangle\right] \geq 0.
\end{equation*}
We denote by $\hat\partial f(x)$ the set of regular subdifferential of $f$ at $x$. 
Moreover, a vector $v \in \mathbb{R}^n$ is called a general subdifferential of $f$ at $x\in\mathrm{dom}(f)$ if there exist a sequence $\{x_m\}_{m=1}^\infty$ such that $ x_m \rightarrow x$ with $f\left(x_m\right) \rightarrow f(x)$ and $v_m \in \hat{\partial} f\left(x_m\right)$ with $v_m\to v$, as $m\to\infty$. We denote by $\partial f(x)$ the set of general subdifferential of $f$ at $x$.

We recall several useful properties of the general subdifferential. If $f$ is finite at $x$ and $g$ is differentiable on a neighborhood of $x$, then there holds 
$$
\partial (f+g)(x)=\partial f(x)+\nabla g(x).
$$
For any proper, lower semicontinuous and convex function $f:\mathbb{R}^n\to\mathbb{R}\cup\{+\infty\}$ and any point $x\in\mathrm{dom}(f)$, there holds 
\begin{equation*}\label{Classical_Subdifferential}
    \partial f(x)=\left\{v\in\mathbb{R}^n : f(y) \geq f(x)+\langle v, y-x\rangle \text { for all } y\in\mathbb{R}^n\right\}.
\end{equation*}
In other words, in such a case,  $\partial f(x)$ coincides with the classical subdifferential of convex function $f$. 
The generalized Fermat rule states that if a proper function $f:\mathbb{R}^n\to\mathbb{R}\cup\{+\infty\}$ has a local minimum at $x$, then $0\in\partial f(x)$.

We are ready to present the relation between the parameter choice and the sparsity of a local minimizer of problem \eqref{general single lambda model}. 
\begin{theorem}\label{prop: single choice_sparsity_general_smooth}
Suppose that  $\mathcal{L}$ is differentiable on $\mathbb{R}^t$. If problem \eqref{general single lambda model} with $\lambda>0$ has a local minimizer $\Theta^{*}:=(w^*,b^*)$ with $w^*:=\sum_{i\in\mathbb{S}_{d_W,l^*}}w_{i}^*{e}_{i}\in \Omega_{d_W,l^*}$ for some $l^*\in\mathbb{Z}_{d_W+1}$, %
then %
\begin{equation}\label{single lambda_general_smooth1}
\lambda=-\nabla_{w_{i}}\mathcal{L}(\Theta^*)\mathrm{sign}\left(w_{i}^*\right), \
i\in \mathbb{S}_{d_W,l^*},
\end{equation}
and
\begin{equation}\label{single lambda_general_smooth11}
\lambda\geq\left|\nabla_{w_i}\mathcal{L}(\Theta^*)\right|, \ i\in \mathbb{N}_{d_W}\setminus\mathbb{S}_{d_W,l^*}.
\end{equation}
\end{theorem}
\begin{proof}
    By setting $\mathcal{R}(\Theta):=\lambda\|w\|_1$ for $\Theta:=(w,b)\in\mathbb{R}^t$, we represent problem \eqref{general single lambda model} as 
    \begin{equation} \label{single lambda model R theta}
    \min\left\{\mathcal{L}(\Theta) + \mathcal{R}(\Theta):\Theta\in\mathbb{R}^t \right\}.
    \end{equation}
    It follows from the generalized Fermat rule that if $\Theta^*$ is a local minimizer of problem \eqref{single lambda model R theta} then there holds
     \begin{equation}\label{single 0 in subgradient}
    0\in\partial (\mathcal{L}+\mathcal{R})(\Theta^*).    
    \end{equation}
    Since $\mathcal{L}$ is differentiable, 
    we have that 
    $$
    \partial (\mathcal{L}+\mathcal{R})(\Theta^*)=\nabla\mathcal{L}(\Theta^*)+\partial\mathcal{R}(\Theta^*).
    $$ 
   By the definition of $\mathcal{R}$, it is convex, and thus, $\partial\mathcal{R}(\Theta^*)$ is the classical subdifferential of the convex function $\mathcal{R}$. Substituting the above equation into the right-hand side of inclusion relation \eqref{single 0 in subgradient} yields that 
   $$
   0\in\nabla\mathcal{L}(\Theta^*)+\partial\mathcal{R}(\Theta^*).
   $$
   Noting that function $\mathcal{R}$ is independent of the bias vectors, we get from the above inclusion relation that 
    \begin{equation*}
    0\in\nabla_{w}\mathcal{L}(\Theta^*)+ \lambda\partial\|\cdot\|_1(w^*),
    \end{equation*}
    which further leads to 
    \begin{equation}\label{single 0 inclusion relation}
    -\nabla_{w}\mathcal{L}(\Theta^*)\in\lambda\partial\|\cdot\|_1(w^*).
    \end{equation} 
    By hypothesis, we have that 
    \begin{equation}\label{Sparse_Vector}
         w^*=\sum_{i\in\mathbb{S}_{d_W,l^*}}w_{i}^*{e}_{i}\in \Omega_{d_W,l^*}\ \  \mbox{with}\ \ w_{i}^*\in\mathbb{R}\setminus\{0\},\ i\in\mathbb{S}_{d_W,l^*},
    \end{equation}
    with which we further observe that 
    \begin{equation}\label{single subdifferential of l1 norm}
    \partial\|\cdot\|_1\left(w^*\right)=\left\{z\in\mathbb{R}^{d_W}: z_{i}=\mathrm{sign}(w_{i}^*), i\in\mathbb{S}_{d_W,l^*} \ \mbox{and}\ |z_i|\leq 1,\ i\in\mathbb{N}_{d_W}\setminus\mathbb{S}_{d_W,l^*}\right\}.
    \end{equation}
    Combining inclusion relation \eqref{single 0 inclusion relation} with equation \eqref{single subdifferential of l1 norm}, for $i\in \mathbb{S}_{d_W,l^*}$, we find that 
    $$
    -\nabla_{w_{i}}\mathcal{L}(\Theta^*)=\lambda \mathrm{sign}(w_{i}^*),
    $$
    which yields equation  \eqref{single lambda_general_smooth1}, and for $i\in\mathbb{N}_{d_W}\setminus\mathbb{S}_{d_W,l^*}$, we obtain that
    $
    |\nabla_{w_{i}}\mathcal{L}(\Theta^*)|\leq \lambda,
    $
    which is exactly the inequality \eqref{single lambda_general_smooth11}.
\end{proof}

Theorem \ref{prop: single choice_sparsity_general_smooth} affirms that a regularization parameter $\lambda^*$ such that problem \eqref{general single lambda model} has a local minimizer $\Theta^{*}:=(w^*,b^*)$ with a sparse $w^*$ of level $l^*\in\mathbb{Z}_{d_W+1}$, in the form of \eqref{Sparse_Vector}, must satisfies conditions \eqref{single lambda_general_smooth1} and \eqref{single lambda_general_smooth11}.  Components $w^*_i$ that satisfy equality \eqref{single lambda_general_smooth1} are nonzero and components $w^*_i$ that satisfy inequality \eqref{single lambda_general_smooth11} are zero. Therefore, conditions \eqref{single lambda_general_smooth1} and  \eqref{single lambda_general_smooth11} can be used to develop a strategy for choices of the regularization parameter $\lambda$. To this end,  we define 
\begin{equation}\label{a-sequence}
    a_i(\Theta):=\left|\nabla_{w_i}\mathcal{L}(\Theta)\right|, \ \ \mbox{for}\ \ \Theta:=(w,b), \ \ i\in\mathbb{N}_{d_W}. 
\end{equation}
We evaluate $a_i$ at $\Theta^*$ for all $i\in\mathbb{N}_{d_W}$ and rearrange the sequence $a_i(\Theta^*)$, $i\in\mathbb{N}_{d_W}$ in a nondecreasing order such that 
\begin{equation}\label{Ordering}
    a_{i_1}(\Theta^{*})\leq a_{i_2}(\Theta^*)\leq\cdots \leq a_{i_{d_W}}(\Theta^{*}), \ \  \mbox{with}\ \{i_1, i_2, \dots, i_{d_W}\}=\mathbb{N}_{d_W}.
\end{equation}
If $a_{i_k}(\Theta^*)<\lambda^*$, then for sure we have that $w^*_{i_k}=0$. However, if $\lambda^*=a_{i_k}(\Theta^*)$, then the corresponding component $w^*_{i_k}$ may be zero or nonzero.
The following result is derived from Theorem \ref{prop: single choice_sparsity_general_smooth}.

\begin{theorem}\label{theorem: single vector lambda relation-new}
Let $\mathcal{L}$ be differentiable on $\mathbb{R}^t$. Suppose that $\Theta^{*}:=(w^*,b^*)$ is a local minimizer of problem \eqref{general single lambda model} with $\lambda^*$ and $a_{i}(\Theta^*)$, $i\in\mathbb{N}_{d_W}$, are ordered as in \eqref{Ordering}, where $a_i$ are defined by \eqref{a-sequence}. 

(1) If $w^*$ has sparsity of level $l^*\in\mathbb{Z}_{d_W+1}$, then $\lambda^*$ satisfies
\begin{equation}\label{Order}
    a_{i_1}(\Theta^{*})\leq \cdots \leq a_{i_{{d_W}-l^*}}(\Theta^{*})\leq\lambda^*=a_{i_{d_W-l^*+1}}(\Theta^{*})=\cdots=a_{i_{d_W}}(\Theta^{*}).
\end{equation}

(2) If $w^*$ has sparsity of level $l^*\in\mathbb{Z}_{d_W+1}$, then there exists $l\in\mathbb{Z}_{d_W+1}$ with $l\geq l^*$ such that $\lambda^*$ satisfies 
\begin{equation}\label{Order-new}
    a_{i_1}(\Theta^{*})\leq \cdots \leq a_{i_{{d_W}-l}}(\Theta^{*})<\lambda^*=a_{i_{d_W-l+1}}(\Theta^{*})=\cdots=a_{i_{d_W}}(\Theta^{*}).
\end{equation}

(3) If there exists $l\in\mathbb{Z}_{d_W+1}$ such that $\lambda^*$ satisfies 
inequality \eqref{Order-new}, then $w^*$ has sparsity of level $l^*$ with $l^*\leq l$.
\end{theorem}

\begin{proof}
We first prove Item (1). According to Theorem \ref{prop: single choice_sparsity_general_smooth}, if $\Theta^*$ is a local minimizer of problem \eqref{general single lambda model} with a regularization parameter $\lambda^*$ having a sparse $w^*$ with sparsity level $l^*$, in the form of \eqref{Sparse_Vector}, then there are exactly $l^*$ elements of $\{a_i(\Theta^*): i\in\mathbb{N}_{d_W}\}$ equal to $\lambda^*$ and the remaining $d_W-l^*$ elements less than or equal to $\lambda^*$, where $a_i$ are defined by \eqref{a-sequence}. Since $a_i(\Theta^*)$ are ordered as in \eqref{Ordering}, we get the desired inequality \eqref{Order}.
    
We next verify Item (2). As pointed out in Item (1), if $\Theta^{*}:=(w^*,b^*)$ is a local minimizer of problem \eqref{general single lambda model} with $\lambda^*$, where $w^*$ has sparsity of level $l^*\in\mathbb{Z}_{d_W+1}$, then $\lambda^*$ satisfies inequality \eqref{Order}. If $l^*=d_W$ or $l^*<d_W$,  $a_{i_1}(\Theta^{*})= \cdots = a_{i_{{d_W}-l^*}}(\Theta^{*})=\lambda^*$, then inequality \eqref{Order} reduces to $\lambda^*=a_{i_k}(\Theta^{*})$, $k\in\mathbb{N}_{d_W}$. We then get inequality \eqref{Order-new} with $l:=d_W$. Otherwise, we choose $k\in\mathbb{N}_{d_W-l^*}$ such that $a_{i_{k}}(\mathbf{u}^{*})<\lambda^*=a_{ i_{k+1}}(\mathbf{u}^{*})$. By letting $l:=d_W-k$, we rewrite inequality \eqref{Order} as inequality \eqref{Order-new}. It is clear that $l\geq d_W-(d_W-l^*)=l^*$.

It remains to show Item (3). For the case when $l<d_W$, the relation $a_{i_1}(\Theta^{*})\leq \cdots \leq a_{i_{{d_W}-l}}(\Theta^{*})<\lambda^*$, guaranteed by Theorem \ref{prop: single choice_sparsity_general_smooth}, shows that  $w^{*}_{i_k}=0,$ for all $k\in\mathbb{N}_{d_W-l}$. As a result,  $w^{*}$ has at least $d_W-l$ number of zero components. In other words, the number of nonzero components of $w^{*}$ is at most $l$, that is, $w^*$ has sparsity of level $l^*$ with $l^*\leq l$. For the case when $l=d_W,$ it is clear that the sparsity level $l^*$ of $w^*$ satisfies $l^*\leq l$. 
\end{proof}

We comment that  inequality \eqref{Order} of Theorem \ref{theorem: single vector lambda relation-new} is a necessary condition for $w^*$ to have sparsity of level $l^*$, but it may not be a sufficient condition. In Item (2), we modify   inequality \eqref{Order} to inequality  \eqref{Order-new} so that it becomes a sufficient condition as shown in Item (3) of Theorem \ref{theorem: single vector lambda relation-new}.

We next propose an iterative scheme, based on Theorem \ref{theorem: single vector lambda relation-new}, for choice of the regularization parameter $\lambda$ that ensures the resulting neural network achieves a target sparsity level (TSL) $l^*$. According to Theorem \ref{theorem: single vector lambda relation-new}, if problem \eqref{general single lambda model} with the regularization parameter $\lambda^*$ has a local minimizer $\Theta^*$ with $w^*$ having exactly $l^*$ nonzero components, then $\lambda^*$ must satisfy inequality \eqref{Order}. Given a TSL $l^*$, our aim is to find a number $\lambda^*>0$ and a local minimizer $\Theta^*$ of problem \eqref{general single lambda model} with the regularization parameter $\lambda^*$ that satisfy  inequality \eqref{Order}. This will be done by an iteration that alternatively updates the regularization parameter and the local minimizer of the corresponding minimization problem. Specifically, we pick a regularization parameter and solve problem \eqref{general single lambda model} with the regularization parameter to obtain its local minimizer. Then, according to the sparsity level of the computed local minimizer, we alter the regularization parameter and solve problem \eqref{general single lambda model} with the updated regularization parameter. At each step of this iteration, we are required to solve the nonconvex minimization problem \eqref{general single lambda model}. Unlike the regularization parameter choice strategy introduced in \cite{liu2023parameter} where by solving a convex optimization problem at each step, one can obtain a good approximation of its local minimizer using a deterministic convergence guaranteed algorithm, here we are required to solve highly nonconvex minimization problem \eqref{general single lambda model} with a large number of network parameters and large amount of training data. The  nonconvex minimization problem \eqref{general single lambda model} is often solved by stochastic algorithms, which introduce uncertainty. Such uncertainty makes it challenging to approximate a local minimizer accurately. Recent advancements in multi-grade deep learning \cite{xu2023multi,xu2023sal} may help mitigate the difficulties in the training process. Even though, stochastic algorithms are used in solving optimization problems involved in training. %
Influence of uncertainty introduced by stochastic algorithms makes it difficult to match the exact sparsity level as the strategy developed in \cite{liu2023parameter} does for learning problems involving a convex fidelity term. For this reason, we do not require exact match of sparsity levels and instead, we allow them to have a tolerance error. Let $l$ denote the total number of nonzero weight parameters of a neural network. With a given tolerance $\epsilon>0$, we say that the neural network achieves a TSL $l^*$ if 
\begin{equation}\label{layer stop single parameter}
    \left|l-l^*\right|/l^*\leq\epsilon.
\end{equation}
The number $\epsilon$ serves as a stopping criteria for the iterative scheme to be described.

We now describe an iterative scheme for choice of the regularization parameter $\lambda$ that ensures the resulting neural network achieves a TSL $l^*$.
The iteration begins with two initial regularization parameters $0<\lambda^2<\lambda^1$. We pick $\lambda^1$ large enough so that the resulting sparsity level $l^1$ of the weight matrices of the corresponding neural network is smaller than the given TSL $l^*$, and  $\lambda^2$ small enough so that the resulting sparsity level $l^2$ of the weight matrices of the corresponding neural network exceeds $l^*$. Let $\Theta^{2}:=(w^2,b^2)$ denote a local minimizer of problem \eqref{general single lambda model} with $\lambda^2$ and $w^2$ has a sparsity level $l^2$. We then evaluate $a_i$ at $\Theta^2$ for all $i\in\mathbb{N}_{d_W}$ and rearrange the resulting sequence in a nondecreasing order 
\begin{equation}\label{lambda2}
  a_{i_1}(\Theta^{2})\leq a_{i_2}(\Theta^2)\leq\cdots \leq a_{i_{d_W}}(\Theta^{2}).  
\end{equation}
We update the regularization parameter according to \eqref{lambda2}. If there is no element of the sequence $\{a_i^2:=a_i(\Theta^2): i\in\mathbb{N}_{d_W}\}$ belonging to the interval $(\alpha, \beta):=(\lambda^2,\lambda^1)$, we choose $\lambda^3:=(\lambda^1+\lambda^2)/2$. Otherwise, we suppose that 
\begin{equation}\label{lambda2_NEW}
  a_{i_1}^{2}\leq \cdots\leq a_{i_{p-1}}^2\leq\alpha<a_{i_p}^2\leq\cdots \leq a_{i_{p+\mu}}^{2}< \beta\leq a_{i_{p+\mu+1}}^2\leq\cdots\leq a_{i_{d_W}}^{2}  
\end{equation}
and choose $\lambda^3$ as the median of $\{a_{i_p}^2, \dots,  a_{i_{p+\mu}}^{2}\}$. We solve problem \eqref{general single lambda model} with $\lambda^3$ and obtain its local minimizer $\Theta^{3}:=(w^3,b^3)$ with $w^3$ having the sparsity level $l^3$. If $l^3$ satisfies condition \eqref{layer stop single parameter}, then the iteration terminates with the desired parameter and local minimizer. Otherwise, we evaluate $a_i$ at $\Theta^3$ for all $i\in\mathbb{N}_{d_W}$ and rearrange the sequence in a nondecreasing order 
\begin{equation}\label{lambda3}
  a_{j_1}(\Theta^{3})\leq a_{j_2}(\Theta^3)\leq\cdots \leq a_{j_{d_W}}(\Theta^{3}).  
\end{equation}
We update the regularization parameter according to \eqref{lambda3}. If $l^3<l^*$, we set $\alpha:=\lambda^2$, $\beta:=\lambda^3$ and if $l^3>l^*$, we set $\alpha:=\lambda^3$, $\beta:=\lambda^1$. We choose $\lambda^4$ in a way similar to that for $\lambda^3$. If there is no element of the sequence $\{a_i^3:=a_i(\Theta^3): i\in\mathbb{N}_{d_W}\}$ belonging to the interval $(\alpha, \beta)$, we choose $\lambda^4:=(\alpha+\beta)/2$. Otherwise, we suppose that $\alpha$ and $\beta$ satisfy \eqref{lambda2_NEW} with 2 replaced by 3 and the sequence $i_k$, $k\in \mathbb{N}_{d_W}$, replaced by  the sequence $j_k$, $k\in \mathbb{N}_{d_W}$, and we choose  $\lambda^4$ as the median of $\{a^3_{j_p},\dots, a^3_{j_{p+\mu}}\}$.  
We repeat the above procedure to obtain an iterative algorithm, with which we obtain a desirable regularization parameter and a sparse neural network achieving the TSL $l^*$  simultaneously.

We summarize in Algorithm \ref{algo: iterative scheme picking single lambda} the iterative scheme for choice of the regularization parameter. Numerical experiments to be presented in
section 5 demonstrate that Algorithm \ref{algo: iterative scheme picking single lambda} can effectively choose a regularization parameter that leads to a local minimizer of problem \eqref{general single lambda model} with weight matrices having desired sparsity level and satisfactory approximation accuracy. 

\begin{algorithm}
  \caption{Iterative scheme selecting single regularization parameter for model \eqref{general single lambda model}}
  
  \label{algo: iterative scheme picking single lambda}
  
  \KwInput{$\mathcal{L}$, $l^*$, $\epsilon$}
  
  \KwInitialization{Choose $\lambda^1$ large enough that guarantees $l^1<l^*$. Choose $\lambda^2$ small enough that guarantees $l^2> l^*$.}
  \For{$i = 2,3,\ldots$}
  {Solve \eqref{general single lambda model} with $\lambda^i$ and get the corresponding numerical solution $\Theta^i=(w^i,b^i)$. Let $l^{i}$ be the sparsity level of $w^i$. \\
  \If{$|l^i-l^*|/l^*\leq\epsilon$} {
    \textbf{break}  %
  }
  
  Compute $a^i:=\left|\nabla_{w^i}\mathcal{L}(\Theta^i)\right|$.\\
  Obtain $j_1$ such that $l^{j_{1}}=\max\{l^j: l^j\leq l^*, j\in\mathbb{N}_i\}$.\\
  Obtain $j_2$ such that 
  $l^{j_{2}}=\min\{l^j: l^*\leq l^j, j\in\mathbb{N}_i\}$.\\
  \uIf{\text{there exists $j$ such that} $\lambda_k^{j_2}<(a^i)_j<\lambda_k^{j_1}$}{
  Update $\lambda^{i+1}$ as the median of $\left\{(a^i)_j:\lambda^{j_2}<(a^i)_j<\lambda^{j_1}\right\}$.}
  \Else{Update $\lambda^{i+1}$ as $(\lambda^{j_2}+\lambda^{j_1})/2$.}
  
  }

  \KwOutput{$\Theta^i$, $\lambda^i$.}  
\end{algorithm}

\section{Multi-Parameter Regularization} \label{section: multi parameter model}
In this section, we explore multi-parameter $\ell_1$-norm regularization for training sparse DNNs. Specifically, we impose different regularization parameters to weights of different layers. This approach is intuitive for DNNs since each layer captures different types of physical information. We derive such regularization models for regression and classification from the Bayes perspective and characterize how the regularization parameter in a layer influences the sparsity of the weight matrix of the layer. Based on the characterization, we then develop an iterative scheme for selecting multiple regularization parameters with which the weight matrix of each layer enjoys a prescribed sparsity level.

We begin with extending the idea of Proposition \ref{prop: square loss single lambda} to derive a multi-parameter regularization model for regression problem from the Bayes Theorem. Unlike in the single-parameter  regularization model where all the weight parameters follow the Laplace distribution with the same scale parameter $s>0$, we allow the weight parameters of different layers in the next model to follow the Laplace
distribution with possibly different scale parameters. As shown in the next proposition, this will lead to a multi-parameter regularization model.

\begin{proposition}\label{theorem: multi lambda regression}
    Suppose that $\left\{\left(x^i,y^i\right)\in\mathbb{R}^p\times\mathbb{R}^q: i\in\mathbb{N}_N\right\}$ is a given dataset and the labels $y^i$ are the observed values of random variables $\tilde{y}^i$ defined by equation \eqref{epsilon j tilde} with $\epsilon_j^i$, $(i,j)\in\mathbb{N}_N\times\mathbb{N}_q$, independently following $\mathrm{Gaussian} (0,v^2)$ with $v>0$ and DNNs $\mathcal{N}_{\Theta}$ satisfy equation \eqref{epsilon j} for parameters $\Theta:=(w^k,b^k)_{k\in\mathbb{N}_D}$ with $w^k\in\mathbb{R}^{d_k}$, $b^k\in\mathbb{R}^{n_k}$, $k\in\mathbb{N}_D$. If for each $k\in\mathbb{N}_D$, $w^k_j$, $j\in\mathbb{N}_{d_k}$ independently follow $\mathrm{Laplace}(0,s_k)$ with $s_k>0$, and $b^k_j$, $j\in\mathbb{N}_{n_k}$ independently follow $\mathrm{Uniform}(-M,M)$ with $M>0$, then the MAP estimate $\Theta^*$ is a solution of the following optimization problem
    \begin{equation}\label{multi lambda model regression}
    \min\left\{\sum_{i\in\mathbb{N}_N}\left\|\mathcal{N}_{\Theta}(x^i)-y^i\right\|_2^2 + \sum_{k\in\mathbb{N}_D}\lambda_k\|w^k\|_1:\Theta:=(w^k,b^k)_{k\in\mathbb{N}_D}\in\mathbb{R}^t \right\}
    \end{equation}
    with $\lambda_k:=2v^2/s_k$, $k\in\mathbb{N}_D.$
\end{proposition}

\begin{proof} We prove this proposition by modifying the proof of Proposition  \ref{prop: square loss single lambda}.
Let $Y:=[y^i_j:i\in\mathbb{N}_N,j\in\mathbb{N}_q]\in\mathbb{R}^{q\times N}$ and consider the random variable $\widetilde{Y}:=[\widetilde{y}^i_j:i\in\mathbb{N}_N,j\in\mathbb{N}_q]\in\mathbb{R}^{q\times N}$. It is known that the MAP estimate $\Theta^*$ is a solution of problem \eqref{optimization problem two ln}. As shown in the proof of Proposition \ref{prop: square loss single lambda with bias regularized}, the probability $\log(p(\widetilde{Y}=Y|\Theta))$ has the form \eqref{ln p(y|Theta)}. It suffices to compute the probability $\log(p(\Theta))$ with the form \eqref{w and b are independent}. Note that for each $k\in\mathbb{N}_D$, $w^k_j$, $j\in\mathbb{N}_{d_k}$ independently follow $\mathrm{Laplace}(0,s_k)$ with $s_k>0$, and $b^k_j$, $j\in\mathbb{N}_{n_k}$ independently follow $\mathrm{Uniform}(-M,M)$ with $M>0$. As a result, we obtain by the probability density functions \eqref{pdf of Laplace} and \eqref{pdf of Uniform} of Laplace and uniform distributions, respectively, that 
\begin{equation*}
    \log(p(\Theta))=\log\left(\prod_{k\in\mathbb{N}_D}\frac{1}{(2M)^{n_k}}\left(\prod_{j\in\mathbb{N}_{d_k}}\frac{1}{2s_k}\mathrm{exp}\left(-\frac{|w^k_j|}{s_k}\right)\right)\right),
\end{equation*}
which further leads to 
\begin{equation}\label{multi ln p Theta}
\log(p(\Theta))=C_4 - \sum_{k\in\mathbb{N}_D} \frac{1}{s_k} \|w^k\|_1,
\end{equation}
where 
$$
C_4:=-\log(2M)\sum_{k\in\mathbb{N}_D} n_k - \sum_{k\in\mathbb{N}_D} d_k\log(2s_k).
$$
Note that $C_1$ in \eqref{ln p(y|Theta)} and $C_4$ defined as above are both constants
independent of $\Theta$. Substituting equations \eqref{ln p(y|Theta)} and \eqref{multi ln p Theta} into
the object function of problem \eqref{optimization problem two ln}, we get that $\Theta^*$ is a solution of the optimization problem \begin{equation}\label{proof argmax equiv}
    \max\left\{-\frac{1}{2v^2}\sum_{i\in\mathbb{N}_N}\left\|\mathcal{N}_{\Theta}(x^i)-y^i\right\|_2^2- \sum_{k\in\mathbb{N}_D}\frac{1}{s_k}\|w^k\|_1:\Theta:=(w^k,b^k)_{k\in\mathbb{N}_D}\in\mathbb{R}^t\right\}. 
\end{equation}
 By multiplying $-2v^2$ on each term of the object function of problem  \eqref{proof argmax equiv} and setting $\lambda_k:=2v^2/s_k$, $k\in\mathbb{N}_D$, we see that  problem \eqref{proof argmax equiv} is equivalent to optimization problem \eqref{multi lambda model regression}. This completes the proof of this proposition. 
\end{proof}

Unlike the regularization term of optimization problem \eqref{single lambda regression model}, which penalizes the weights of all layers together with one single parameter, that of optimization problem \eqref{multi lambda model regression} involves $D$ parameters $\lambda_j$, $j\in\mathbb{N}_D$, penalizing the weights of different layers with different regularization parameters.

The next result pertains to the multi-parameter regularization model for classification problem. Similarly to Proposition \ref{theorem: multi lambda regression}, we impose  prior Laplace distribution with distinct scale parameters on the weight parameters of various layers.

\begin{proposition}\label{theorem: multi lambda classification}
    Let $K\in\mathbb{N}$. Suppose that $\left\{\left(x^i,y^i\right)\in\mathbb{R}^p\times\mathbb{N}_K: i\in\mathbb{N}_N\right\}$ is a given dataset and the labels $y^i$ are the observed values of random variables $\tilde{y}^i$ defined by equation \eqref{classification tilde yi true} and DNNs $\mathcal{N}_{\Theta}$ satisfy equation \eqref{classification tilde yi given} for parameters $\Theta:=(w^k,b^k)_{k\in\mathbb{N}_D}$ with $w^k\in\mathbb{R}^{d_k}$, $b^k\in\mathbb{R}^{n_k}$, $k\in\mathbb{N}_D$. If for each $k\in\mathbb{N}_D$, $w^k_j$, $j\in\mathbb{N}_{d_k}$ independently follow $\mathrm{Laplace}(0,s_k)$ with $s_k>0$, and $b^k_j$, $j\in\mathbb{N}_{n_k}$ independently follow $\mathrm{Uniform}(-M,M)$ with $M>0$, then the MAP estimate $\Theta^*$ is a solution of the following optimization problem
    \begin{equation}\label{multi lambda model classification}
    \min\left\{-\sum_{i\in\mathbb{N}_N} \log\left(\left(\mathcal{S}(\mathcal{N}_{\Theta}(x^i))\right)_{y^i}\right) + \sum_{k\in\mathbb{N}_D}\lambda_k\|w^k\|_1:\Theta:=(w^k,b^k)_{k\in\mathbb{N}_D} \in\mathbb{R}^t\right\}
    \end{equation}
    with $\lambda_k:=1/s_k$, $k\in\mathbb{N}_D$. 
\end{proposition}

\begin{proof} We prove this proposition by computing the two probabilities involved in maximization problem \eqref{optimization problem two ln}.
As shown in section 2, the MAP estimate $\Theta^*$ is the solution of maximization problem \eqref{optimization problem two ln} with $Y:=[y^i:i\in\mathbb{N}_N]\in\mathbb{N}_K^{N}$ and the random variable $\widetilde{Y}:=[\widetilde{y}^i:i\in\mathbb{N}_N]\in\mathbb{N}_K^{N}$. Under the assumptions of this proposition, the probabilities   $\log(p(\widetilde{Y}=Y|\Theta))$ and $\log(p(\Theta))$ constituting the object function of problem \eqref{optimization problem two ln} have been computed in the proofs of Propositions \ref{prop: single lambda classification} and \ref{theorem: multi lambda regression}, respectively. That is, they are given by \eqref{proof log p classification} and \eqref{multi ln p Theta}, respectively. By plugging these two representations into the object function of problem \eqref{optimization problem two ln}, with noting that $C_4$ in equation \eqref{multi ln p Theta} is a constant independent of $\Theta$, we may reformulate maximization problem \eqref{optimization problem two ln} as minimization problem \eqref{multi lambda model classification} with $\lambda_k:=1/s_k$, $k\in\mathbb{N}_D$. 
\end{proof}

We remark that in models \eqref{multi lambda model regression} and \eqref{multi lambda model classification}, the regularization parameter $\lambda_k$ for the $k$-th layer may vary as $k$ changes.

We now turn to studying
choices of multiple regularization parameters that guarantee desired sparsity levels of the regularized solutions. To this end, we consider a general multi-parameter regularization model, which covers models \eqref{multi lambda model regression} and \eqref{multi lambda model classification} as special cases. Suppose that $\mathcal{L}:\mathbb{R}^t\to\mathbb{R}$ is a general loss function of variable $\Theta\in\mathbb{R}^t$. The multi-parameter regularization model is formulated as 
\begin{equation}\label{general multi lambda model}
    \min\left\{\mathcal{L}(\Theta) + \sum_{k\in\mathbb{N}_D}\lambda_k\|w^k\|_1:\Theta:=(w^k,b^k)_{k\in\mathbb{N}_D}\in\mathbb{R}^t \right\},
\end{equation}
where $\lambda_k>0$, $k\in\mathbb{N}_D$, are a group of regularization parameters.

The next result elucidates the relationship between a set of regularization parameters and the sparsity level
of a local minimizer of problem \eqref{general multi lambda model}. We recall that for the single-parameter regularization problem, Theorem \ref{prop: single choice_sparsity_general_smooth} shows how the single regularization parameter $\lambda$ determines the sparsity level of the weight matrices. The multiple regularization parameters allow us to consider the sparsity level of the weight matrix of each layer separately. 

\begin{theorem}\label{prop: multi choice_sparsity_general_smooth}
Suppose that $\mathcal{L}$ is differentiable on $\mathbb{R}^t$. If problem \eqref{general multi lambda model} with $\lambda_k>0$, $k\in\mathbb{N}_D$ has a local minimizer $\Theta^{*}:=\left(w^{k*},b^{k*}\right)_{k\in\mathbb{N}_D}$ with for each $k\in\mathbb{N}_D$,  $w^{k*}:=\sum_{i\in\mathbb{S}_{d_k,l_k^*}}w^{k*}_{i}{e}_{i}\in \Omega_{d_k,l_k^*}$ for some $l_k^*\in\mathbb{Z}_{d_k+1}$, then for each $k\in\mathbb{N}_D$
\begin{equation}\label{multi lambda_general_smooth1}
\lambda_k=-\nabla_{w^k_{i}}\mathcal{L}(\Theta^*)\mathrm{sign}\left(w^{k*}_{i}\right), \ \ i\in \mathbb{S}_{d_k,l_k^*},
\end{equation}
and
\begin{equation}\label{multi lambda_general_smooth11}
    \lambda_k\geq\left|\nabla_{w^k_{i}}\mathcal{L}(\Theta^*)\right|, \ \ i\in \mathbb{N}_{d_k}\setminus\mathbb{S}_{d_k,l_k^*}.
\end{equation}
\end{theorem}

\begin{proof}
We prove this theorem by employing arguments similar to those used in the proof of Theorem \ref{prop: single choice_sparsity_general_smooth}. We denote the regularization term in problem \eqref{general multi lambda model} by  
$$
\mathcal{R}_{\Lambda}(\Theta):=\sum_{k\in\mathbb{N}_D}\lambda_k\|w^k\|_1, \ \ \mbox{for}\ \ \Theta:=(w^k,b^k)_{k\in\mathbb{N}_D}\in\mathbb{R}^t.
$$
The generalized Fermat's rule ensures that if $\Theta^*$ is a local minimizer of problem \eqref{general multi lambda model}, then we get that 
\begin{equation*}\label{Fermat's_Rule}
    0\in\partial (\mathcal{L}+\mathcal{R}_{\Lambda})(\Theta^*),
\end{equation*}
which together with the differentiability of $\mathcal{L}$ leads to  
\begin{equation}\label{Fermat's_Rule:Conponents}
0\in\nabla\mathcal{L}(\Theta^*)+\partial\mathcal{R}_{\Lambda}(\Theta^*).
\end{equation}
Note that $\partial\mathcal{R}_{\Lambda}(\Theta^*)$ is the classical subdifferential due to the convexity of function $\mathcal{R}_{\Lambda}$. According to the definition of $\mathcal{R}_{\Lambda}$, we rewrite inclusion relation \eqref{Fermat's_Rule:Conponents} as
\begin{equation*}
0\in\nabla_{w^k}\mathcal{L}(\Theta^*)+ \lambda_k\partial\|\cdot\|_1(w^{k*}),\ k\in\mathbb{N}_D,
\end{equation*}
which is equivalent to 
\begin{equation}\label{0 inclusion relation}
-\nabla_{w^k}\mathcal{L}(\Theta^*)\in\lambda_k\partial\|\cdot\|_1(w^{k*}),\ k\in\mathbb{N}_D.
\end{equation} 
For each $k\in\mathbb{N}_D$, by noting that 
\begin{equation*}\label{representation_w^{k*}} w^{k*}=\sum_{i\in\mathbb{S}_{d_k,l_k^*}} w^{k*}_{i}{e}_{i}\in \Omega_{d_k,l_k^*}\ \  \mbox{with}\  \ w^{k*}_{i}\in\mathbb{R}\setminus\{0\},\ \ i\in\mathbb{S}_{d_k,l_k^*},
\end{equation*}
we obtain that 
\begin{equation}\label{subdifferential of l1 norm}
\partial\|\cdot\|_1\left(w^{k*}\right)
=\left\{z\in\mathbb{R}^{d_k}: z_{i}:=\mathrm{sign}(w^{k*}_{i}),i\in\mathbb{S}_{d_k,l_k^*}, 
|z_i|\leq 1, i\in\mathbb{N}_{d_k}\setminus\mathbb{S}_{d_k,l_k^*}\right\}.
\end{equation}
Substituting equation  \eqref{subdifferential of l1 norm} into the right-hand side of inclusion relation  \eqref{0 inclusion relation}, we get that $-\nabla_{w^k_i}\mathcal{L}(\Theta^*)=\lambda_k\mathrm{sign}(w^{k*}_{i})$ for all $i\in\mathbb{S}_{d_k,l_k^*}$, 
which is equivalent to equation \eqref{multi lambda_general_smooth1} and $\left|\nabla_{w^k_{i}}\mathcal{L}(\Theta^*)\right|\leq\lambda_k$ for all $i\in \mathbb{N}_{d_k}\setminus\mathbb{S}_{d_k,l_k^*}$, which coincides with inclusion relation \eqref{multi lambda_general_smooth11}. 
\end{proof}

Theorem \ref{prop: multi choice_sparsity_general_smooth} characterizes the sparsity level of the weight matrix of a layer in terms of the regularization parameter in the layer. That is, if $\Theta^{*}:=\left(w^{k*},b^{k*}\right)_{k\in\mathbb{N}_D}$ is a local minimizer of problem \eqref{general multi lambda model} with $\lambda_k>0$, $k\in\mathbb{N}_D$ and for each $k\in\mathbb{N}_D$,  $w^{k*}$ has sparsity of level $l_k^*\in\mathbb{Z}_{d_k+1}$, then each regularization parameter $\lambda_k$ satisfies equality \eqref{multi lambda_general_smooth1} and inequality \eqref{multi lambda_general_smooth11}. Based upon this characterization, we develop an iterative scheme for choosing parameter $\lambda_k>0$, $k\in\mathbb{N}_D$, with which a local minimizer of problem \eqref{general multi lambda model} has a prescribed sparsity level in each layer. For this purpose, we introduce a sequence for each layer according to equality \eqref{multi lambda_general_smooth1} and inequality \eqref{multi lambda_general_smooth11}.
For
each $k\in\mathbb{N}_D$, we set \begin{equation}\label{aik}
a^k_i(\Theta):=\left|\nabla_{w^k_i}\mathcal{L}(\Theta)\right|, \ \mbox{for}\ \Theta:=(w^k,b^k)_{k\in\mathbb{N}_D},\ i\in\mathbb{N}_{d_k}.
\end{equation}
We evaluate $a^k_i$ at $\Theta^*$ 
for all $i\in\mathbb{N}_{d_k}$ and rearrange the sequence  $a^k_i(\Theta^*)$, $i\in\mathbb{N}_{d_k}$, in a nondecreasing order: 
\begin{equation}\label{nondecreasing order}
a^k_{i_1}(\Theta^{*})\leq\cdots \leq a^k_{i_{d_k}}(\Theta^{*}),\ \mbox{with}\ \{i_1,i_2,\ldots,i_{d_k}\}=\mathbb{N}_{d_k}.
\end{equation}
In a way similar to the single-parameter regularization model, these sequences so defined will be used to choose the regularization parameters in each step of the iterative scheme. The following result may be proved by arguments similar to those used in the proof of Theorem \ref{theorem: single vector lambda relation-new}, and thus details of the proof are omitted.

\begin{theorem}\label{theorem: multi vector lambda relation-new}
   Let $\mathcal{L}$ be differentiable on $\mathbb{R}^t$.  Suppose that $\Theta^{*}:=\left(w^{k*},b^{k*}\right)_{k\in\mathbb{N}_D}$ is a local minimizer of problem \eqref{general multi lambda model} with $\lambda_k^*>0$, $k\in\mathbb{N}_D$, and for each $k\in\mathbb{N}_D$,  $a^{k}_{i}(\Theta^*)$, $i\in\mathbb{N}_{d_k}$, are ordered as in \eqref{nondecreasing order}, where $a^{k}_{i}$ are defined by \eqref{aik}.
   
   (1) If for each $k\in\mathbb{N}_D$, $w^{k*}$ has sparsity of level  $l_k^*\in\mathbb{Z}_{d_k+1}$, then 
   for each $k\in\mathbb{N}_D$, $\lambda_k^*$ satisfies 
\begin{equation*}\label{multi Order}
    a^k_{i_1}(\Theta^{*})\leq \cdots \leq a^k_{i_{{d_k}-l_k^*}}(\Theta^{*})\leq\lambda_k^*=a^k_{i_{{d_k}-l_k^*+1}}(\Theta^{*})=\cdots=a^k_{i_{d_k}}(\Theta^{*}).
\end{equation*} 
   
   (2) If for each $k\in\mathbb{N}_D$, $w^{k*}$ has sparsity of level  $l_k^*\in\mathbb{Z}_{d_k+1}$, then for each $k\in\mathbb{N}_D$, there exists $l_k\in\mathbb{Z}_{d_k+1}$ with $l_k\geq l_k^*$ such that   $\lambda_k^*$ satisfies 
\begin{equation}\label{multi Order-new}
    a^k_{i_1}(\Theta^{*})\leq \cdots \leq a^k_{i_{{d_k}-l_k}}(\Theta^{*})<\lambda_k^*=a^k_{i_{{d_k}-l_k+1}}(\Theta^{*})=\cdots=a^k_{i_{d_k}}(\Theta^{*}).
\end{equation}  

(3) If for each $k\in\mathbb{N}_D$, there exists $l_k\in\mathbb{Z}_{d_k+1}$
such that $\lambda_k^*$ satisfies
inequality \eqref{multi Order-new}, then for each $k\in\mathbb{N}_D$, $w^{k*}$ has sparsity of level $l_k^*\leq l_k$.
\end{theorem}

 Based on Theorem \ref{theorem: multi vector lambda relation-new}, we propose an iterative scheme for selecting regularization parameters $\lambda_k$,  $k\in\mathbb{N}_D$, such that the weight matrices of the resulting neural network achieve prescribed target sparsity levels (TSLs) $\{l_k^*\}_{k\in\mathbb{N}_D}$. This will be done by extending Algorithm \ref{algo: iterative scheme picking single lambda} to the multi-parameter case. Unlike Algorithm \ref{algo: iterative scheme picking single lambda}, where we choose the single regularization parameter that penalizes all weight parameters, in the iterative scheme for choosing the multiple regularization parameters, we choose different parameters for different layers. In fact, Theorem \ref{theorem: multi vector lambda relation-new} allows us to choose the regularization parameter $\lambda_k$ of the $k$-th layer according to the sequence $a^k_i$, $i\in\mathbb{N}_{d_k}$ in a way similar to that in Algorithm \ref{algo: iterative scheme picking single lambda}.
 
 Due to possible huge amount of trainable parameters and the interplay between multiple regularization parameters across different layers, we do not seek to precisely align the weights in neural networks with the  TSLs for each layer. Instead, we adjust our objective from attaining exact sparsity levels to finding neural networks whose  number of nonzero weights closely matches the TSLs within a given tolerance. We achieve this by extending the tolerance \eqref{layer stop single parameter} to the multi-parameter case. Let $l_k$  denote the total number of nonzero weight parameters in the $k$-th layer of a neural network. With a given tolerance $\epsilon>0$ and a threshold $K\in\mathbb{N}_D$, we say that the neural network achieves TSLs $\{l_k^*\}_{k\in\mathbb{N}_D}$ if 
\begin{equation*}\label{layer stop}
    \left|\left\{k\in\mathbb{N}_D:\left|l_k-l_k^*\right|/l_k^*\leq\epsilon\right\}\right|\geq K,
\end{equation*}
where $|A|$ refers to the cardinality of a finite set $A$. Here, $\epsilon$ and $K$ serve as parameters for stopping criteria, tailoring to specifics of individual learning tasks.

We summarize the iterative scheme for choosing the multiple regularization parameters in Algorithm \ref{algo: iterative scheme picking lambda}. Numerical experiments to be presented in Section \ref{section: numerical experiment} demonstrate the effectiveness of this algorithm in identifying the desired regularization parameters. 
\begin{algorithm}
  \caption{Iterative scheme selecting multiple regularization parameters for model \eqref{general multi lambda model}}
  
  \label{algo: iterative scheme picking lambda}
  
  \KwInput{$\mathcal{L}$, $\{l_k^*\}_{k\in\mathbb{N}_D}$, $\epsilon$, $K$}
  
  \KwInitialization{Choose $\{\lambda_k^1\}_{k\in\mathbb{N}_D}$ large enough that guarantees $l_k^1<l_k^*$ for all $k\in\mathbb{N}_D$. Choose $\{\lambda_k^2\}_{k\in\mathbb{N}_D}$ small enough that guarantees $l_k^2>l_k^*$ for all $k\in\mathbb{N}_D$.}
  \For{$i = 2,3,\ldots$}
  {Solve \eqref{general multi lambda model} with $\{\lambda_k^i\}_{k\in\mathbb{N}_D}$ and get the corresponding numerical solution $\Theta^i=((w^i)^k,(b^i)^k)_{k\in\mathbb{N}_D}$.\\
  Let $l_k^{i}$ be the sparsity level of $(w^i)^k$ for each $k\in\mathbb{N}_D$.\\
  \If{$\left|\left\{k\in\mathbb{N}_D:|l_k^i-l_k^*|/l_k^*\leq\epsilon\right\}\right|\geq K$} {
    \textbf{break}  %
  }
  
  \For{$k = 1,2,\ldots,D$}
  {
  Compute $(a^i)^k:=\left|\nabla_{(w^i)^k}\mathcal{L}(\Theta^i)\right|$.\\
  Obtain $j_1$ such that $l_k^{j_{1}}=\max\{l_k^j: l_k^j\leq l_k^*, j\in\mathbb{N}_i\}$.\\
  Obtain $j_2$ such that 
  $l_k^{j_{2}}=\min\{l_k^j: l_k^*\leq l_k^j, j\in\mathbb{N}_i\}$.\\

  \uIf{\text{there exists $j$ such that} $\lambda_k^{j_2}< (a^i)^k_j< \lambda_k^{j_1}$}{
  Update $\lambda_k^{i+1}$ as the median of $\left\{(a^i)^k_j:\lambda_k^{j_2}<(a^i)^k_j< \lambda_k^{j_1}\right\}$.}
  \Else{Update $\lambda_k^{i+1}$ as $(\lambda_k^{j_2}+\lambda_k^{j_1})/2$.}
  
  }
  }
  
  \KwOutput{$\Theta^i$, $\{\lambda_k^i\}_{k\in\mathbb{N}_D}$.}  
\end{algorithm}

\section{Proximal Gradient Descent Algorithm}\label{ImplementationIssues}
When implementing Algorithms \ref{algo: iterative scheme picking single lambda} and \ref{algo: iterative scheme picking lambda}, we need to solve optimization problems \eqref{general single lambda model} and  \eqref{general multi lambda model} repeatedly. These problems are non-convex and have non-differentiable regularization terms. Having an efficient algorithm for solving these optimization problems is essential for the success of implementing Algorithms \ref{algo: iterative scheme picking single lambda} and \ref{algo: iterative scheme picking lambda}. We describe in this section a mini-batch proximal gradient descent algorithm for solving these optimization problems. Proximal algorithms for solving convex non-differential optimization problems were developed in \cite{Li-Shen-Xu-Zhang:AiCM:15,Micchelli_2011}.

We recall the definition of proximity operators. Suppose that $f:\mathbb{R}^n\to \overline{\mathbb{R}}:= \mathbb{R}\cup\{+\infty\}$ is a convex function, with $\mathrm{dom}(f):=\{x\in\mathbb{R}^n:f(x)<+\infty\}\neq{\emptyset}.$ The proximity operator $\text{prox}_{f}:\mathbb{R}^n\to\mathbb{R}^n$ of $f$ is defined by
$$
\text{prox}_{f}(x):=\argmin\left\{\frac{1}{2}\|u-x\|_2^2+f(u):u\in\mathbb{R}^n\right\}\ \ \mbox{for}\ x\in\mathbb{R}^n. 
$$
Particularly, when $f$ is chosen by the $\ell_1$-norm  multiplied by a positive constant coefficient $\beta$, namely $f(x)=\beta\|x\|_1$, $x\in\mathbb{R}^n$, then for each $x=[x_i:j\in\mathbb{N}_n]\in\mathbb{R}^n$, the proximity operator of $f$ is given by $\mathrm{prox}_{\beta\|\cdot\|_1}(x)=[u_i:i\in\mathbb{N}_n]\in\mathbb{R}^n$ with 
\begin{equation*}
    u_i:=\mathrm{sign}(x_i)\max\{|x_i|-\beta,0\},
\end{equation*}
where sign denotes the sign function, assigning $-1$ for a negative input and $1$ for a non-negative input. This closed-form formula for $\mathrm{prox}_{\beta\|\cdot\|_1}$ facilitates fast solution of the optimization problem \eqref{general multi lambda model}. 

The mini-patch proximal gradient descent algorithm is a useful approach to solving problem \eqref{general multi lambda model} with regularization parameters $\{\lambda_k\}_{k\in\mathbb{N}_D}$. To describe the algorithm, we specify several parameters such as the learning rate $\alpha$, epoch number ``$E$", and mini-batch size ``$B$". During each epoch, we randomly shuffle the training dataset and divide  it into several mini-batches of a given mini-batch size, except that the last mini-batch may have less size. Suppose that there are $I$ mini-batches, then for each $i\in\mathbb{N}_I$, the $i$-th mini-batch and its size are denoted by $G_i$ and $|G_i|$, respectively. For each $i\in\mathbb{N}_I$, we denote by $\mathcal{L}_{G_i}:\mathbb{R}^t\to\mathbb{R}$ the loss function associated with the mini-batch $G_i$. Then, for each mini-batch, we update layer-wisely the bias parameters through gradient descent and the weight parameters through proximal gradient descent. We check if the resulting neural network satisfies the desirable condition after the update of each mini-batch while running Algorithm \ref{algo: iterative scheme picking lambda}. We present the multi-parameter mini-batch proximal gradient descent algorithm in Algorithm \ref{algo: Mini-batch proximal gradient descent algorithm}.
\begin{algorithm}
  \caption{Multi-parameter mini-batch proximal gradient descent algorithm}
  
  \label{algo: Mini-batch proximal gradient descent algorithm}
  
  \KwInput{$\{\lambda_k\}_{k\in\mathbb{N}_D}$, $\alpha$, $E$, $B$}

  \For{$e = 1, 2,\ldots,E$}{
      Randomly shuffle the training dataset and obtain mini-batches $G_i$, $i\in\mathbb{N}_I$. 
      
      \For{$i = 1, 2,\ldots, I$}{
      
        \For{$k = 1, 2,\ldots, D$}{
        $w^k\leftarrow \mathrm{prox}_{\alpha\lambda_k\|\cdot\|_1}\left(w^k-\alpha\cdot\frac{1}{|G_i|}\nabla_{w^k}\mathcal{L}_{G_i}(\Theta)\right)$\\
        $b^k\leftarrow b^k-\alpha\cdot\frac{1}{|G_i|}\nabla_{b^k}\mathcal{L}_{G_i}(\Theta)$
        
        }
      }

  }
    
\end{algorithm}

\section{Numerical Experiments}\label{section: numerical experiment}
In this section, we present numerical results to validate  the parameter choice strategies established in Sections \ref{section: parameter choice for single, sparsity and accuracy} and \ref{section: multi parameter model}. Application problems studied in this section include regression and classification.

All experiments presented in this section are performed on the First Gen ODU HPC Cluster, and implemented with Pytorch \cite{paszke2019pytorch}. The computing jobs are randomly placed on an X86\_64 server with the computer nodes Intel(R) Xeon(R) CPU E5-2660 0 @ 2.20GHz (16 slots), Intel(R) Xeon(R) CPU E5-2660 v2 @ 2.20GHz (20 slots), Intel(R) Xeon(R) CPU E5-2670 v2 @ 2.50GHz (20 slots), Intel(R) Xeon(R) CPU E5-2683 v4 @ 2.10GHz (32 slots).

In all the numerical examples presented in this section, we choose the rectified linear unit (ReLU) function 
\begin{equation*}\label{relu}
    \mathrm{ReLU}(x):=\max\{0,x\},\ x\in\mathbb{R},
\end{equation*}
as the activation function for each layer of a neural network. Note that ReLU is not differentiable at $x=0$ and in the numerical experiments while computing gradients in Algorithms \ref{algo: iterative scheme picking single lambda} and \ref{algo: iterative scheme picking lambda}, we simply set its derivative at $x=0$ to be $0$. 

In presentation of the numerical results, we use ``TSL" and ``TSLs" to represent the target sparsity level $l^*$ in Algorithm \ref{algo: iterative scheme picking single lambda} and layer-wise target sparsity level $\{l_k^*\}_{k\in\mathbb{N}_D}$ in Algorithm \ref{algo: iterative scheme picking lambda}, respectively. We use ``SL'' to denote the sparsity level of the weight matrices of the neural network obtained from Algorithm \ref{algo: iterative scheme picking single lambda} and ``SLs'' to denote the layer-wise sparsity level of the weight matrices of the neural network obtained from Algorithm \ref{algo: iterative scheme picking lambda}. By ``Ratio'', we denote the ratio of the number of the nonzero weight parameters to the total number of weight parameters in the obtained neural network. We also use ``NUM'' to denote the number of iterations for obtaining the desired regularization parameter whose corresponding solution aligns with the stopping criteria when running Algorithms \ref{algo: iterative scheme picking single lambda} and \ref{algo: iterative scheme picking lambda}.

\subsection{Regression}
In this subsection, we consider parameter choice for regression problems: the single-parameter choice of model \eqref{single lambda regression model} implemented by Algorithm \ref{algo: iterative scheme picking single lambda} and the multi-parameter choice of model \eqref{multi lambda model regression} implemented by Algorithm \ref{algo: iterative scheme picking lambda}. 
For the comparison purpose, we also consider the single-parameter $\ell_2$ regularization model for regression and take the resulting approximation accuracy as a reference to compare with that obtained by Algorithms \ref{algo: iterative scheme picking single lambda} and \ref{algo: iterative scheme picking lambda}. 

We first describe the neural network used in this experiment. We choose fully connected neural network with $3$ hidden layers (totally $4$ layers) as the neural network architecture in models \eqref{single lambda regression model} and \eqref{multi lambda model regression}. The width of each hidden layer of the neural network is specified by $128$, $128$ and $64$, respectively. In this neural network, the number of weight parameters for each layer is $
[d_k:k\in\mathbb{N}_4]=[768, 16384, 8192, 64]
$ and the total number of the weight parameters is $d_W=25408.$ 

We test the proposed algorithms by using the benchmark dataset ``Mg"  \cite{chang2011libsvm,flake2002}. The dataset is composed of time series data generated by the Mackey-Glass system, which contains $1385$ instances, each with  $6$ features. The $6$ features represent lagged observations of the Mackey-Glass time series, each capturing a past value at a specific time delay. The regression task is to predict the next value of the Mackey-Glass time series based on the $6$ lagged observations. We take $N:=1000$ instances as training samples and $N':=385$ instances as testing samples. For given training samples $\left\{\left(x^i,y^i\right)\in\mathbb{R}^6\times\mathbb{R}: i\in\mathbb{N}_N\right\}$, we define the mean square error (MSE) by
\begin{equation}\label{MSE}
    \mathrm{MSE}:=\frac{1}{N}\sum_{i\in\mathbb{N}_N} \left(\mathcal{N}_{\Theta}(x^i)-y^i\right)^2.
\end{equation} 
Moreover, we denote by $\mathrm{TrMSE}$ and $\mathrm{TeMSE}$ the $\mathrm{MSE}$ on the training dataset and the testing dataset, respectively. $\mathrm{TeMSE}$ is computed by using \eqref{MSE} with $N$ replaced by $N'$ and the training samples replaced by the test samples.

For the comparison purpose, we train the neural network by the single-parameter $\ell_2$ regularization model
\begin{equation}\label{classical DNN model}
    \min\left\{\sum_{i\in\mathbb{N}_N}\left\|\mathcal{N}_{\Theta}(x^i)-y^i\right\|_2^2+\lambda\|w\|_2^2:\Theta:=(w,b)\in\mathbb{R}^t\right\}.
\end{equation}
We choose the parameter $\lambda$ as $0$, $10^{-5}$, $10^{-4}$, $10^{-3}$, $10^{-2}$, and solve model \eqref{classical DNN model} with each of the regularization parameters. We list the numerical results in Table \ref{table: regression result L2}, from which we see that all the neural networks obtained are dense, and the best TrMSE and TeMSE values are $0.96\times 10^{-2}$ and $1.44\times 10^{-2}$, respectively, when $\lambda=10^{-4}$. We will take this  approximation accuracy as a reference for the neural network obtained by Algorithms \ref{algo: iterative scheme picking single lambda} and \ref{algo: iterative scheme picking lambda}. 

\begin{table}
\captionsetup{justification=centering}
\caption{\label{table: regression result L2}The single-parameter $\ell_2$ regularization model for regression}
\vspace{-0.5cm}
\begin{center}
\begin{tabular}{cccccc}
\hline\hline
$\lambda$ & $0$      & $10^{-5}$ & $10^{-4}$ & $10^{-3}$ & $10^{-2}$ \\ \hline
TrMSE   & $0.90\times10^{-2}$ & $0.92\times10^{-2}$  & $\mathbf{0.96\times10^{-2}}$  & $1.27\times10^{-2}$  & $21.21\times10^{-2}$  \\ 
TeMSE  & $1.53\times10^{-2}$ & $1.52\times10^{-2}$  & $\mathbf{1.44\times10^{-2}}$  & $1.44\times10^{-2}$  & $2.02\times10^{-2}$  \\ 
Ratio     & $100\%$ & $100\%$        & $\mathbf{100\%}$        & $100\%$  & $100\%$    \\ \hline\hline
\end{tabular}
\end{center}
\end{table}

We employ Algorithm \ref{algo: iterative scheme picking single lambda} to select a single regularization parameter that ensures the resulting neural network achieves a given TSL. For two prescribed TSL  values $2000$ and $5000$, we apply Algorithm \ref{algo: iterative scheme picking single lambda} with $\epsilon=0.1\%$ and initialize $\lambda^1=10^{-3}$, $\lambda^2=10^{-7}$ to find desired regularization parameters $\lambda^*$ and a corresponding local minimizer $\Theta^{*}:=\left(w^{k*},b^{k*}\right)_{k\in\mathbb{N}_4}$ of  model \eqref{single lambda regression model}. When running Algorithm \ref{algo: iterative scheme picking single lambda}, we solve \eqref{general single lambda model} by employing Algorithm \ref{algo: Mini-batch proximal gradient descent algorithm} with epoch $E=50000$, learning rate $\alpha=0.1$, and the full batch size. Numerical results for this experiment are reported in Table \ref{table: regression result single L1}. For the two TSL values, the algorithm reaches the stopping criteria within $5$ and $4$ iterations, respectively. The sparsity level SL of the neural network obtained by Algorithm \ref{algo: iterative scheme picking single lambda} 
coincide with the TSL within  tolerance error $\epsilon=0.1\%$.
The resulting sparse DNNs with only $7.88\%$ and $19.69\%$ of nonzero weight parameters achieve an approximation accuracy comparable to those shown in Table \ref{table: regression result L2}, with even smaller TeMSE values.

\begin{table} 
    \captionsetup{justification=centering}
    \caption{Single-parameter choices (Algorithm \ref{algo: iterative scheme picking single lambda}) for model  \eqref{single lambda regression model}}
    \label{table: regression result single L1}
    \vspace{-0.5cm}
    \begin{center}
    \begin{tabular}{ccc}
    \hline\hline
    TSL        & $2000$              & $5000$              \\ \hline
    $\lambda^*$ & $7.24\times 10^{-5}$ & $3.87\times10^{-5}$ \\ 
    SL       & $2001$              & $5002$              \\ 
    Ratio       & $7.88\%$                            & $19.69\%$                            \\ 
    NUM         & $5$                                 & $4$                                  \\ 
    TrMSE       & $1.26\times10^{-2}$                             & $1.18\times10^{-2}$                             \\ 
    TeMSE   & $1.43\times10^{-2}$                             & $1.40\times10^{-2}$                             \\ \hline\hline
    \end{tabular}
    \end{center}
\end{table}

\begin{table} 
    \captionsetup{justification=centering}
    \caption{\label{table: regression result multi L1}Multi-parameter choices (Algorithm \ref{algo: iterative scheme picking lambda}) for model \eqref{multi lambda model regression}}
    \vspace{-0.5cm}
    \begin{center}
    \begin{tabular}{ccc}
    \hline\hline
    TSLs        & $[650, 1200, 1000, 45]$              & $[700, 3000, 1500, 50]$              \\  \hline
    $\{\lambda_k^*\}_{k\in\mathbb{N}_4}$ & $[4.19,5.80,1.87,2.70]\times10^{-5}$ & $[1.66,3.34,2.37,1.72]\times10^{-5}$ \\ 
    SLs        & $[553, 1258, 1036, 43]$              & $[713, 3730, 1538, 52]$              \\ 
    Ratio       & $11.37\%$                            & $23.74\%$                            \\ 
    NUM         & $10$                                 & $4$                                  \\ 
    TrMSE       & $1.14\times10^{-2}$                             & $1.12\times10^{-2}$                             \\ 
    TeMSE   & $1.42\times10^{-2}$                             & $1.39\times10^{-2}$                             \\ \hline\hline
    \end{tabular}
    \end{center}
\end{table}

We then employ Algorithm \ref{algo: iterative scheme picking lambda} to choose multiple regularization parameters with which the resulting neural network achieves given TSLs. For two prescribed TSLs values $[650,1200,1000,45]$ and $[700,3000,1500,50]$, we apply Algorithm \ref{algo: iterative scheme picking lambda} with assigning $\epsilon=5\%$, $K=3$ and initialize $\lambda^1=10^{-3}$, $\lambda^2=10^{-7}$ to find desired regularization parameters $\{\lambda_k^*\}_{k\in\mathbb{N}_4}$ and a corresponding local minimizer $\Theta^{*}:=\left(w^{k*},b^{k*}\right)_{k\in\mathbb{N}_4}$ of the multi-parameter $\ell_1$ regularization model \eqref{multi lambda model regression}. When running Algorithm \ref{algo: iterative scheme picking lambda}, we solve \eqref{general multi lambda model} by employing Algorithm \ref{algo: Mini-batch proximal gradient descent algorithm} with epoch $E=50000$, learning rate $\alpha=0.1$, and the full batch size. We report the numerical results of this experiment in Table \ref{table: regression result multi L1}. 
For the two TSLs values, the algorithm reaches the stopping criteria within $10$ and $4$ iterations, respectively. The layer-wise sparsity level SLs of the neural network obtained by Algorithm \ref{algo: iterative scheme picking lambda}
match with the TSLs within  tolerance error $\epsilon=5\%$. The resulting sparse DNNs have $11.37\%$ and $23.74\%$ nonzero weight parameters. The numerical results in Table \ref{table: regression result multi L1} demonstrates the efficacy of Algorithm \ref{algo: iterative scheme picking lambda} in selecting regularization parameters to achieve a desired sparsity level in DNNs, with preserving the approximation accuracy shown in Table \ref{table: regression result L2} for the dense DNN.

\subsection{Classification}
We consider in this subsection the parameter choice strategy for classification problems:  the single-parameter choice of model \eqref{single lambda classification model} implemented by Algorithm \ref{algo: iterative scheme picking single lambda} and the multi-parameter choice of model \eqref{multi lambda model classification} implemented by Algorithm \ref{algo: iterative scheme picking lambda}. The single-parameter $\ell_2$ regularization model for classification is also considered as for the comparison purpose.

We begin with describing the convolutional neural network (CNN) used in this experiment. We construct a CNN comprising three convolutional layers (Conv-1, Conv-2, Conv-3) and two fully connected layers (FC-1, FC-2). The filter mask in all three convolutional layers is set to be size $3\times 3$ and stride $1$ with padding $1$. Each convolutional layer is succeeded  by a maximum pooling layer (MaxPool) and activated using the ReLU function. The maximum pooling window is set to be size $2\times2$ and stride $2$. The architectural details of the CNN are outlined in Table \eqref{table: CNN}. %
The number of weight parameters for convolutional and fully connected layers in this CNN architecture is $
[d_k:k\in\mathbb{N}_5]=[288, 18432, 73728, 589824, 5120]$, and the total number of the weight parameters is $d_W=687392$. 

\begin{table}
\captionsetup{justification=centering}
\caption{\label{table: CNN}Structure of CNN}
\vspace{-0.5cm}
\begin{center}
\begin{tabular}{cc}
\hline\hline
Layers     & Output size of each layer     \\ \hline
$\text{Input}\in\mathbb{R}^{28\times28}$   &     \\ 
$\text{Conv-1}:   1\times(3\times3), 32$   & $32\times(28\times28)$        \\ 
MaxPool: $2\times2$,   stride=$2$          & $32\times(14\times14)$ (ReLU) \\ 
$\text{Conv-2}:   32\times(3\times3), 64$  & $64\times(14\times14)$        \\ 
MaxPool: $2\times2$,   stride=$2$          & $64\times(7\times7)$ (ReLU)   \\ 
$\text{Conv-3}:   64\times(3\times3), 128$ & $128\times(7\times7)$         \\ 
MaxPool, $2\times2$,   stride=$2$          & $128\times(3\times3)$ (ReLU)  \\ 
Flatten                                    & $1152$                        \\ 
FC-1: $1152\times 512$                     & $512$ (ReLU)                  \\ 
FC-2: $512\times 10$                       & $10$ (Softmax)                \\ 
$\text{Output}\in[0,1]^{10}$          &                               \\ \hline\hline
\end{tabular}
\end{center}
\end{table}

We employ the CNN to classify image samples in MNIST database \cite{lecun1998gradient}. This dataset comprises $60000$ samples for training  and $10000$ for testing. Each sample is a grayscale image of size $28\times 28$ pixels depicting handwritten digits from ``0'' to ``9''. We evaluate the model's accuracy by counting labels correctly predicted for both the training and testing datasets, denoted by  $\mathrm{TrA}$ and $\mathrm{TeA}$, respectively.

To facilitate comparison, we train the neural network using the single-parameter $\ell_2$ regularization model
\begin{equation}\label{l2 classification DNN model}
    \min\left\{-\sum_{i\in\mathbb{N}_N} \log\left(\left(\mathcal{S}(\mathcal{N}_{\Theta}(x^i))\right)_{y^i}\right)+\lambda\|w\|_2^2:\Theta:=(w,b)\in\mathbb{R}^t\right\}.
\end{equation}
We experiment with various values of the regularization parameter $\lambda$: $0$, $10^{-5}$, $10^{-4}$, $10^{-3}$, $10^{-2}$, and summarize the numerical results in Table \ref{table: classification result L2}. As shown in Table \ref{table: classification result L2}, the neural networks obtained from the single-parameter $\ell_2$ regularization model \eqref{l2 classification DNN model} are dense. Among the tested values, $\lambda=10^{-4}$ yields the neural network with the best outcome. We will use this level of accuracy as a reference for evaluating the neural networks generated by Algorithms \ref{algo: iterative scheme picking single lambda} and \ref{algo: iterative scheme picking lambda}.

\begin{table}
\captionsetup{justification=centering}
\caption{\label{table: classification result L2}The single-parameter $\ell_2$ regularization model for classification}
\vspace{-0.5cm}
\begin{center}
\begin{tabular}{cccccc}
\hline\hline
$\lambda$ & $0$       & $10^{-5}$ & $10^{-4}$ & $10^{-3}$ & $10^{-2}$ \\ \hline
TrA       & $100\%$   &  $100\%$  & $\mathbf{100\%}$   & $99.63\%$ & $98.05\%$ \\ 
TeA       & $99.30\%$ &   $99.28\%$  & $\mathbf{99.38\%}$ & $99.26\%$ & $98.07\%$ \\
Ratio     & $100\%$ &   $100\%$  & $\mathbf{100\%}$ & $100\%$ & $100\%$ \\ \hline\hline
\end{tabular}
\end{center}
\end{table}

We implement Algorithm \ref{algo: iterative scheme picking single lambda} to select the regularization parameter with which the
resulting neural network achieves a given TSL. For two prescribed  TSL values $50000$ and $100000$, we apply Algorithm \ref{algo: iterative scheme picking single lambda} with $\epsilon=0.1\%$ and initialize $\lambda^1=10^{-2}$, $\lambda^2=10^{-7}$ to find the desired regularization parameter $\lambda^*$ and a corresponding local minimizer $\Theta^{*}:=\left(w^{k*},b^{k*}\right)_{k\in\mathbb{N}_5}$. During running Algorithm \ref{algo: iterative scheme picking single lambda}, we solve \eqref{general single lambda model} using Algorithm \ref{algo: Mini-batch proximal gradient descent algorithm} with  epoch $E=200$, learning rate $\alpha=0.1$, and mini-batch size $B=128$. The numerical results are reported in Table \ref{table: classification result single L1}, from which we observe that Algorithm \ref{algo: iterative scheme picking single lambda} meets the stopping criteria after $6$ and $5$ iterations, for TSL being 50000 and 100000, respectively. The resulting SL in the table is close to the TSL. The results in Table \ref{table: classification result single L1} demonstrate the effectiveness of Algorithm \ref{algo: iterative scheme picking single lambda} for selecting the desired regularization parameter with preserving the accuracy shown in Table \ref{table: classification result L2}.

\begin{table}
    \captionsetup{justification=centering}
    \caption{\label{table: classification result single L1}Single-parameter choices (Algorithm \ref{algo: iterative scheme picking single lambda}) for model  \eqref{single lambda classification model}}
    \vspace{-0.5cm}
    \begin{center}
        \begin{tabular}{ccc}
            \hline\hline
            TSL        & $50000$       & $100000$      \\ \hline
            $\lambda^*$ & $6.82\times10^{-5}$ & $1.21\times10^{-5}$ \\ 
            SL  &  $50016$   &  $100029$  \\            
            Ratio       & $7.28\%$                                  & $14.55\%$                                 \\ 
            NUM         & $6$                                      &   $5$                                     \\ 
            TrA         & $99.81\%$                                   & $99.99\%$                                   \\ 
            TeA         & $99.21\%$                                 & $99.13\%$                                 \\ \hline\hline
            \end{tabular}
         \end{center}
\end{table}

Once again, we employ Algorithm \ref{algo: iterative scheme picking lambda} to select the layer-wise regularization parameters in the multi-parameter $\ell_1$ regularization model \eqref{multi lambda model classification}. For two prescribed  TSLs values $[250,8000,13000,20000,4200]$ and $[270,10000,20000,40000,4300]$, we apply Algorithm \ref{algo: iterative scheme picking lambda} with $\epsilon=10\%$, $K=5$ and initialize $\lambda^1=10^{-4}$, $\lambda^2=10^{-7}$ to find desired regularization parameters $\lambda_k^*$, $k\in\mathbb{N}_5$ and a corresponding local minimizer $\Theta^{*}:=\left(w^{k*},b^{k*}\right)_{k\in\mathbb{N}_5}$ of the multi-parameter $\ell_1$ regularization model \eqref{multi lambda model classification}. When running Algorithm \ref{algo: iterative scheme picking lambda}, we solve \eqref{general multi lambda model} using Algorithm \ref{algo: Mini-batch proximal gradient descent algorithm} with epoch $E=200$, learning rate $\alpha=0.1$, and mini-batch size $B=128$. We list the results for this experiment in Table \ref{table: classification result multi L1}, from which we see that Algorithm \ref{algo: iterative scheme picking lambda} meets the stopping criteria within $15$ and $9$ iterations for the two TSLs values, respectively. Additionally,  the number of non-zero weight parameters takes account $6.61\%$ and $11.38\%$ of the dense DNNs. This validates the efficacy of Algorithm \ref{algo: iterative scheme picking lambda} for obtaining  regularization parameters leading to DNNs with desired sparsity level and the accuracy
shown in Table \ref{table: classification result L2}. 

\begin{table}
    \captionsetup{justification=centering}
    \caption{\label{table: classification result multi L1}Multi-parameter choices (Algorithm \ref{algo: iterative scheme picking lambda}) for model \eqref{multi lambda model classification}}
    \vspace{-0.5cm}
    \begin{center}
        \begin{tabular}{ccc}
            \hline\hline
            TSLs        & $[250, 8000, 13000, 20000,   4200]$       & $[270, 10000, 20000, 40000,   4300]$      \\ \hline
            $\lambda^*$ & $[3.77,1.62,1.09,3.93,0.18]\times10^{-5}$ & $[3.82,1.43,9.65,2.21,0.18]\times10^{-5}$ \\ 
            SLs        & $[227, 7984, 14263, 18684, 4296]$         & $[243, 10374, 21964, 41125, 4489]$        \\ 
            Ratio       & $6.61\%$                                  & $11.38\%$                                 \\ 
            NUM         & $15$                                      & $9$                                       \\ 
            TrA         & $100\%$                                   & $100\%$                                   \\ 
            TeA         & $99.12\%$                                 & $99.35\%$                                 \\ \hline\hline
            \end{tabular}
         \end{center}
\end{table}

\section{Conclusions} \label{section: conclusion}

We have  derived sparse deep learning models with single and multiple regularization parameters from a statistical perspective. We have characterized the sparsity level of learned neural networks in terms of the choice of regularization parameters.  Based on the characterizations, we have developed iterative schemes to determine multiple regularization parameters such that the weight parameters in the corresponding deep neural network enjoy prescribed layer-wise sparsity levels. We have tested the proposed algorithms for both regression and classification problems. Numerical results have demonstrated the efficacy of the algorithms in choosing desired regularization parameters that lead to learned neural networks with weight parameters having a predetermined sparsity level and preserving approximation accuracy.

\bigskip
 
\noindent{\bf Acknowledgment:} L. Shen is supported in part by the US National
Science Foundation under grant DMS-2208385. R. Wang is supported in part by the Natural Science Foundation of China under grant 12171202. Y. Xu is supported in part by the US National Science Foundation under grant DMS-2208386, and by the US National Institutes of Health under grant R21CA263876.

\end{document}